\documentclass{l4dc2024}

\title[Bounded Robustness in Reinforcement Learning]{Bounded Robustness in Reinforcement Learning \\
via Lexicographic Objectives}
\usepackage{times}
\usepackage{algorithmic}
\usepackage{booktabs}
\usepackage{wrapfig}

\usepackage{algorithm}
\usepackage{verbatim}
\usepackage{latexsym}
\usepackage{tikz}
\usetikzlibrary{arrows,shapes}
\usepackage{pgfplots}
\usepackage{float}
\usepackage{listings}

\DeclareMathOperator*{\argmax}{argmax}
\DeclareMathOperator*{\argmin}{argmin}

\newtheorem{problem}{Problem}
\newtheorem{assumption}{Assumption}
\author{%
 \Name{Daniel Jarne~Ornia} \Email{d.jarneornia@tudelft.nl}\\
 \addr Delft University of Technology
 \AND
 \Name{Licio Romao} \Email{licio.romao@cs.ox.ac.uk}\\
 \addr University of Oxford%
  \AND
 \Name{Lewis Hammond} \Email{lewis.hammond@cs.ox.ac.uk}\\
 \addr University of Oxford%
  \AND
 \Name{Manuel Mazo~Jr.} \Email{m.mazo@tudelft.nl}\\
 \addr Delft University of Technology%
  \AND
 \Name{Alessandro Abate} \Email{alessandro.abate@cs.ox.ac.uk}\\
 \addr University of Oxford
}
\begin{document}
\maketitle
\begin{abstract}
Policy robustness in Reinforcement Learning may not be desirable at any cost: the alterations caused by robustness requirements from otherwise optimal policies should be explainable, quantifiable and formally verifiable. In this work we study how policies can be \emph{maximally robust} to arbitrary observational noise by analysing how they are altered by this noise through a stochastic linear operator interpretation of the disturbances, and establish connections between robustness and properties of the noise kernel and of the underlying MDPs. Then, we construct sufficient conditions for policy robustness, and propose a robustness-inducing scheme, applicable to any policy gradient algorithm, that formally trades off expected policy utility for robustness through \emph{lexicographic optimisation}, while preserving convergence and sub-optimality in the policy synthesis.
\end{abstract}
\thispagestyle{plain}
\section{Introduction}
Consider a dynamical system where we need to synthesise a controller (policy) through a model-free Reinfrocement Learning \citep{sutton2018reinforcement} approach. When using a simulator for training we expect the deployment of the controller in the real system to be affected by different sources of noise, possibly not predictable or modelled (\emph{e.g.} for networked components we may have sensor faults, communication delays, \emph{etc}). In safety-critical systems, robustness (in terms of successfully controlling the system under disturbances) should preserve formal guarantees, and plenty of effort has been put on developing formal convergence guarantees on policy gradient algorithms \citep{agarwal2021theory,bhandari2019global}. All these guarantees vanish under regularization and adversarial approaches, which are aimed to produce more robust policies. Therefore, for such applications one needs a scheme to regulate the robustness-utility trade-off in RL policies, that on the one hand  preserves the formal guarantees of the original algorithms, and on the other attains sub-optimality conditions from the original problem. Additionally, if we do not know the structure of the disturbance (which holds in most applications), learning directly a policy for an arbitrarily disturbed environment will yield unexpected behaviours when deployed in the true system. 
\paragraph{Lexicographic Reinforcement Learning (LRL)} Recently, lexicographic optimisation \citep{isermann1982linear,rentmeesters1996theory} has been applied to the multi-objective RL setting \citep{Skalse2022}. In an LRL setting some objectives may be more important than others, and so we may want to obtain policies that solve the multi-objective problem in a lexicographically prioritised way, \emph{i.e.}, ``find the policies that optimise objective $i$ (reasonably well), and from those the ones that optimise objective $i+1$ (reasonably well), and so on''.
\paragraph{Previous Work} In robustness against \emph{model uncertainty}, the MDP may have noisy or uncertain reward signals or transition probabilities, as well as possible resulting \emph{distributional shifts} in the training data \citep{heger1994consideration,xu2006robustness,fu2018learning,10.5555/3237383.3238064,pirotta2013safe,abdullah2019wasserstein}, connecting to ideas on distributionally robust optimisation \citep{wiesemann2014distributionally,van2015distributionally}.
For
\emph{adversarial attacks or disturbances} on policies or action selection in RL agents \citep{Gleave2020Adversarial,lin2017tactics,tessler2019action,pan2019risk,tan2020robustifying,10.5555/3306127.3331713, liang2022efficient}, recently \citet{Gleave2020Adversarial} propose to attack RL agents by swapping the policy for an adversarial one at given times. For a detailed review on Robust RL see \citet{moos2022robust}.
Our work focuses in robustness versus \emph{observational disturbances}, where agents observe a disturbed state measurement and use it as input for the policy \citep{kos2017delving,huang2017adversarial,behzadan2017vulnerability,mandlekar2017adversarially,zhang2020robust,zhang2021robust}.
\citet{zhang2020robust} propose a \emph{state-adversarial} MDP framework, and utilise adversarial regularising terms that can be added to different deep RL algorithms to make the resulting policies more robust to observational disturbances, and \citet{zhang2021robust} study how LSTM increases robustness with optimal state-perturbing adversaries.
\paragraph{Contributions} Most existing work on RL with observational disturbances proposes modifying RL algorithms at the cost of \emph{explainability} (in terms of sub-optimality bounds) and \emph{verifiability}, since the induced changes in the new policies result in a loss of convergence guarantees. Our main contributions are summarised in the following points.
\begin{itemize}
    \item We consider general unknown stochastic disturbances and formulate a quantitative definition of observational robustness that allows us to characterise the sets of robust policies for any MDP in the form of operator-invariant sets. We analyse how the structure of these sets depends on the MDP and noise kernel, and obtain an inclusion relation providing intuition into how we can search for robust policies more effectively.\footnote{There are strong connections between Sections \ref{sec:2}-\ref{sec:3} in this paper and the literature on planning for POMDPs \citep{spaan2004point, spaan2012partially} and MDP invariances \citep{ng1999policy,10.5555/3398761.3398926,skalse2022invariance}, as well as recent work concerning robustness misspecification \citep{korkmaz2023adversarial}.}
    \item We propose a meta-algorithm that can be applied to any existing policy gradient algorithm, Lexicographically Robust Policy Gradient (LRPG) that (1) Retains policy sub-optimality up to a specified tolerance while maximising robustness,
    (2) Formally controls the utility-robustness trade-off through this design tolerance, (3) Preserves formal guarantees.
\end{itemize}
Figure \ref{fig:lem1} represents a qualitative interpretation of the results in this work.

\begin{wrapfigure}{r}{0.6\textwidth}
\centering
   \resizebox{0.49\linewidth}{!}{\includegraphics{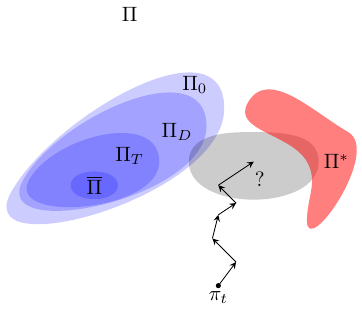}}
\resizebox{0.49\linewidth}{!}{\includegraphics{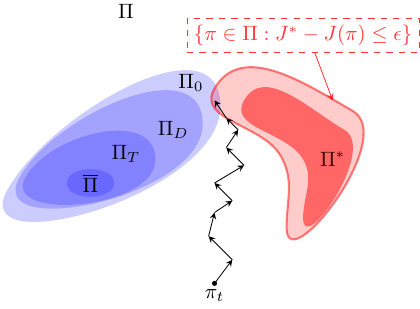}}
\caption{Qualitative representation LRPG (right), compared to usual robustness-inducing algorithms. The sets in blue are the robust policies to be defined in the coming sections. LRPG induces robustness while guaranteeing that the policies will deviate a bounded distance from the optimal.}
\label{fig:lem1}
\vspace{-20pt}
\end{wrapfigure}

\subsection{Preliminaries}
\paragraph{Notation} We use calligraphic letters $\mathcal{A}$ for collections of sets and $\Delta(\mathcal{A})$ as the space of probability measures over $\mathcal{A}$. For two probability distributions $P,P'$ defined on the same $\sigma-$algebra $\mathcal{F}$, $D_{TV}(P\Vert P')=\operatorname{sup}_{A\in \mathcal{F}}|P(A)-P'(A)|$ is the total variation distance. 
For two elements of a vector space we use $\langle\cdot,\cdot\rangle$ as the inner product. We use $\mathbf{1}_n$ as a column-vector of size $n$ that has all entries equal to 1. We say that an MDP is \emph{ergodic} if for any policy the resulting Markov Chain (MC) is ergodic. We say that $S$ is a $n\times n$ row-stochastic matrix if $S_{ij}\geq 0$ and each row of $S$ sums 1. We assume all learning rates in this work $\alpha_t(x,u)\in [0,1]$ ($\beta_t,\eta_t$...) satisfy the conditions $\sum_{t=1}^{\infty}\alpha_t(x,u)=\infty$ and $\sum_{t=1}^{\infty}\alpha_t(x,u)^2<\infty$.
\paragraph{Lexicographic Reinforcement Learning}
Consider a parameterised policy $\pi_\theta$ with $\theta \in \Theta$, and two objective functions $K_1$ and $K_2$. PB-LRL uses a multi-timescale optimisation scheme to optimise $\theta$ faster for higher-priority objectives, iteratively updating the constraints induced by these priorities and encoding them via Lagrangian relaxation techniques \citep{bertsekas1997nonlinear}. Let $\theta' \in \argmax_{\theta} K_1(\theta)$. Then, PB-LRL can be used to find parameters $\theta'' \in \{\argmax_\theta K_2(\theta), 
\, \text{s.t.} \,\, 
K_1(\theta)\geq K_1(\theta')- \epsilon\}.$ This is done through the update:
\begin{equation}\label{eq:LRLupdates}\begin{aligned}
\theta \leftarrow& \operatorname{proj}_{\Theta}\big[\theta +\nabla_\theta \hat{K}(\theta)\big],\quad  \lambda \leftarrow& \operatorname{proj}_{\mathbb{R}_{\geq 0}}\big[ \lambda + \eta_t(\hat{k}_1-\epsilon_t-K_1(\theta))\big],
\end{aligned}
\end{equation}
where $\hat{K}(\theta) := (\beta_t^1 + \lambda \beta_t^2) \cdot K_1(\theta) + \beta_t^2 \cdot K_2(\theta)$, $\lambda$ is a Langrange multiplier, $\beta_t^1,\beta_t^2,\eta_t$ are learning rates, and $\hat{k}_1$ is an estimate of $K_1(\theta')$. Typically, we set $\epsilon_t\to 0$, though we can use other tolerances too, \textit{e.g.}, $\epsilon_t = 0.9 \cdot \hat{k}_1$. For more details see \citet{Skalse2022}.
\section{Observationally Robust Reinforcement Learning}\label{sec:2}
Robustness-inducing methods in model-free RL must address the following dilemma: How do we deal with uncertainty without an explicit mechanism to estimate such uncertainty during policy execution? Consider an example of an MDP where, at policy roll-out phase, there is a non-zero probability of measuring a ``wrong'' state. In such a scenario, measuring the wrong state can lead to executing unboundedly bad actions. This problem is represented by the following version of a noise-induced partially observable Markov Decision Process \citep{spaan2012partially}.
\begin{definition}
An observationally-disturbed MDP (DOMDP) is (a POMDP) defined by the tuple $(X,U,P,R,T,\gamma)$ where $X$ is a finite set of states, $U$ is a set of actions, $P: U\times X\mapsto \Delta(X)$ is a probability measure of the transitions between states and $R:X\times U\times X\mapsto\mathbb{R}$ is a reward function. The map $T:X \mapsto \Delta(X)$ is a stochastic kernel induced by some unknown noise signal, such that $T(y\mid x)$ is the probability of measuring $y$ while the true state is $x$, and acts only on the state observations. At last $\gamma\in[0,1]$ is a reward discount.
\label{def:DOMDP}
\end{definition}
A (memoryless) policy for the agent is a stochastic kernel $\pi: X \mapsto \Delta(U)$. For simplicity, we overload notation on $\pi$, denoting by $\pi(x,u)$ as the probability of taking action $u$ at state $x$.
In a DOMDP\footnote{Definition \ref{def:DOMDP} is a generalised form of the State-Adversarial MDP used by \citet{zhang2020robust}: the adversarial case is a particular form of DOMDP where $T$ assigns probability 1 to one adversarial state.} agents can measure the full state, but the measurement will be disturbed by some  unknown random signal \emph{in the policy deployment}. 
The difficulty of acting in such DOMDP is that agents will have to act based on disturbed states $\tilde{x}\sim T(\cdot \mid x)$. We then need to construct policies that will be as robust as possible against such noise \emph{without the existance of a model to estimate, filter or reject disturbances}.  
The value function of a policy $\pi$ (\emph{critic}), $V^{\pi}:X\mapsto\mathbb{R}$, is given by $V^{\pi}(x_0) = \mathbb{E}[\sum_{t=0}^{\infty}\gamma^t R(x_t,\pi(x_t),x_{t+1}) ]$. The action-value function of $\pi$ ($Q$-function) is given by $Q^{\pi}(x,u)=\sum_{y\in X}P(x,u,y)(R(x,u,y)+\gamma V^{\pi}(y)).$ 
We then define the objective function as $J(\pi):=\mathbb{E}_{x_0\sim \mu_0}[V^\pi(x_0)]$ with $\mu_0$ being a distribution of initial states, and we use $J^*:=\max_{\pi}J(\pi)$ and $\pi^*$ as the optimal policy, and $\Pi^*_{\epsilon}:=\{\pi\in\Pi: J^* -J(\pi)\leq \epsilon\}$ is the set of $\epsilon$-optimal policies. If a policy is parameterised by $\theta \in \Theta$ we write $\pi_\theta$ and $J(\theta)$.
\begin{assumption}\label{as:2}
    For any DOMDP and policy $\pi$, the resulting MC is irreducible and aperiodic.
\end{assumption}
We now formalise a notion of \emph{observational robustness}. Firstly, due to the presence of the stochastic kernel $T$, the policy we are applying is altered as we are applying a collection of actions in a possibly wrong state. Then,
$\langle \pi, T \rangle (x,u) := \sum_{y\in X} T(y \mid x)\pi(y,u),$
where $\langle \pi, T \rangle: X \mapsto \Delta(U)$ is the \emph{disturbed} policy, which averages the current policy given the error induced by the presence of the stochastic kernel. Notice that $\langle \cdot,T\rangle(x):\Pi \mapsto\Delta(U)$ is an averaging operator yielding the alteration of the policy due to noise. We define the \emph{robustness regret}\footnote{The robustness regret satisfies  $\rho(\pi^*,T)\geq 0$ for all kernels $T$, and it allows us to directly compare the robustness regret with the utility regret of the policy.}:
$
\rho(\pi,T) := J(\pi) - J(\langle \pi, T \rangle).
$
\begin{definition}[Policy Robustness]\label{def:rob}
A policy $\pi$ is $\kappa$-\emph{robust} against a stochastic kernel $T$ if
$\rho(\pi,T)\leq \kappa$. If $\pi$ is $0$-\emph{robust} it is {maximally robust}. The sets of $\kappa$-robust policies are $\Pi_{\kappa}:=\{\pi\in\Pi:\rho(\pi,T)\leq \kappa\}$, with $\Pi_0$ being the maximally robust policies.
\end{definition}
One can motivate the characterisation and models above from a control perspective, where policies use as input discretised state measurements with possible sensor measurement errors.
Formally ensuring robustness properties when learning RL policies will, in general, force the resulting policies to deviate from optimality in the undisturbed MDP. We propose then the following problem.
\begin{problem}\label{prob:1}
Consider a DOMDP model as per Definition \ref{def:DOMDP} and let $\epsilon$ be a non-negative tolerance level. Our goal is to find amongst all $\epsilon$-optimal policies those that minimize the robustness level $\kappa$:
\[ \operatorname{minimize} \,\kappa \,\,\, s.t.\, \pi \in \Pi^\star_\epsilon \cap \Pi_{\kappa}.\]
\end{problem}
Note that this is formulated as general as possible with respect to the robustness of the policies: We would like to find a policy that, trading off $\epsilon$ in terms of cumulative rewards, observes the same discounted rewards when disturbed by $T$.
\section{Characterisation of Robust Policies}\label{sec:3}
An important question to be addressed before trying to synthesise robust policies is what these robust policies look like, and how they are related to DOMDP properties. 
A policy $\pi$ is said to be constant if $\pi(x) = \pi(y)$ for all $x, y \in X$, and the collection of all constant policies is denoted by $\bar{\Pi}$. A policy is called a fixed point of $\langle \cdot, T \rangle$ if $\pi(x) = \langle \pi, T \rangle (x) $ for all $x \in X$. The collection of all fixed points is $\Pi_T$.
Observe furthermore that $\Pi_T$ \emph{only depends on the kernel ${T}$} and the set\footnote{There is a (natural) bijection between the set of constant policies and the space $\Delta(U)$. The set of fixed points of the operator $\langle \cdot, T \rangle$ also has an algebraic characterisation in terms of the null space of the operator $\mathrm{Id}(\cdot) - \langle \cdot, T \rangle$. We are not exploiting the later characterisation in this paper.} $X$.
Let us assume we have a policy iteration algorithm that employs an action-value function $Q^\pi$ and policy $\pi$. The advantage function for $\pi$ is defined as $A^\pi (x,u) := Q^\pi (x,u)-V^\pi (x)$. We can similarly define the \emph{noise disadvantage} of policy $\pi$ as:
\begin{equation}\label{eq:valdis}\begin{aligned}
D^\pi(x,T) := V^\pi(x) - \mathbb{E}_{\substack{u \sim \langle \pi, T \rangle (x)}} [ Q^\pi(x,u)],
\end{aligned}
\end{equation}
which measures the difference of applying at state $x$ an action according to the policy $\pi$ with that of playing an action according to $\langle \pi, T \rangle$ and then continuing playing an action according to $\pi$. Our intuition says that if it happens to be the case that $D^\pi(x,T) = 0$ for all states in the DOMDP, then such a policy is maximally robust. And this is indeed the case, as shown in the next proposition.
\begin{proposition}\label{prop:InclusionDisadvantageMaxRobust}
Consider a DOMDP as in Definition \ref{def:DOMDP} and the robustness notion as in Definition \ref{def:rob}. If a policy $\pi$ is such that $D^\pi(x,T)=0 \,\,$ for all $x\in X$, then $\pi$ is maximally robust, i.e., let $\Pi_{D}:=\{\pi\in\Pi:\mu_{\pi} (x)D^\pi(x,T)=0\,\forall\,x\in X\}, $
then we have that $\Pi_D \subseteq \Pi_0$.
\begin{proof}
We want to show that $D^\pi(x,T)=0\implies \rho(\pi,T)=0$. Taking $D^\pi(x,T)=0$ one has a policy that produces an disadvantage of zero when noise kernel $T$ is applied. Then, $\forall \,x\in X,$
\begin{equation}\label{eq:qtilde}
D^\pi(x,T)=0\implies \mathbb{E}_{\substack{u \sim \langle \pi, T \rangle (x)}} [ Q^\pi(x,u)]=V^{\pi}(x).
\end{equation}
Now define the value of the disturbed policy as $V^{\langle \pi, T \rangle}(x)=\mathbb{E}_{\substack{u \sim \langle \pi, T \rangle(x), \\ y \sim P(\cdot \mid x,u)}} \left[ r(x,u,y)+\gamma V^{\langle \pi, T \rangle}(y) \right].$
We will now show that $V^{\pi}(x) = V^{\langle \pi, T \rangle}(x),$ for all $x \in X$. Observe, from \eqref{eq:qtilde} using $V^{\pi}(x)=\mathbb{E}_{\substack{u \sim \langle \pi, T \rangle (x)}} [ Q^\pi(x,u)]$, we have $\forall x\in X$:
\begin{equation}\label{eq:recursive}
    \begin{aligned}
    V^{\pi}(x)-V^{\langle \pi, T \rangle}(x)
    = \mathbb{E}_{\substack{u \sim \langle \pi, T \rangle (x)}} [ Q^\pi(x,u)]-\mathbb{E}_{\substack{u \sim \langle \pi, T \rangle(x) \\ y \sim P(\cdot \mid x,u)}} \left[ r(x,u,y)+\gamma V^{\langle \pi, T \rangle}(y) \right]=\\
    = \mathbb{E}_{\substack{u \sim \langle \pi, T \rangle(x) \\ y \sim P(\cdot \mid x,u)}} \left[\gamma V^{ \pi}(y)-\gamma V^{\langle \pi, T \rangle}(y)\right]= \gamma \mathbb{E}_{y \sim P(\cdot \mid x,u)} \left[ V^{ \pi}(y)- V^{\langle \pi, T \rangle}(y)\right].
    \end{aligned}
\end{equation}
Now, taking the sup norm at both sides of \eqref{eq:recursive} we get
\begin{equation}\label{eq:recursive2}\begin{aligned}
\|V^{\pi}(x)-V^{\langle \pi, T \rangle}(x)\|_\infty =\gamma \left\lVert \mathbb{E}_{y \sim P(\cdot \mid x,u)} \left[ V^{ \pi}(y)- V^{\langle \pi, T \rangle}(y)\right]\right\rVert_\infty.
\end{aligned}
\end{equation}
Since the norm on the right hand side of \eqref{eq:recursive2} is over $y\in X$ and $\gamma<1$, it follows that $\|V^{\pi}(x)-V^{\langle \pi, T \rangle}(x)\|_\infty=0$.
Finally, $\|V^{\pi}(x)-V^{\langle \pi, T \rangle}(x)\|_\infty=0\implies V^{\pi}(x)-V^{\langle \pi, T \rangle}(x)=0\,\,\forall x\in X$, and $V^{\pi}(x)-V^{\langle \pi, T \rangle}(x)=0 \, \forall \,x\in X \implies J(\pi)=J(\langle \pi, T \rangle)\implies \rho(\pi,T)=0$.
\end{proof}
\end{proposition}
So far we have shown that both the set of fixed points $\overline{\Pi}$ and the set of policies for which the disadvantage function is equal to zero $\Pi_{D}$ are contained in the set of maximally robust policies. 
We now show how the defined robust policy sets can be linked in a single result through the following policy inclusions.
\begin{theorem}[Policy Inclusions]\label{theo:InclusionTheorem}
For a DOMDP with noise kernel $T$, consider the sets $\overline{\Pi},\Pi_T,\Pi_{D}$ and $\Pi_0$. Then, the following inclusion relation holds: 
$\overline{\Pi}\subseteq\Pi_{T}\subseteq \Pi_{D} \subseteq \Pi_0.$
Additionally, the sets $\overline{\Pi}, \Pi_{T}$ are \emph{convex} for all MDPs and kernels $T$, but $\Pi_{D}, \Pi_0$ may not be.
\end{theorem}
\begin{proof}
If a policy $\pi\in\Pi$ is a fixed point of the operator $ \langle \cdot,T\rangle$, then $\rho(\pi,T)= J(\pi)-J(\langle \pi,T\rangle)=J(\pi)-J(\pi)=0 \implies \pi\in\Pi_0$. Therefore, $\Pi_T\subseteq \Pi_0$. Now, the space of stochastic kernels $\mathcal{K}:X\mapsto\Delta(X)$ is equivalent to the space of row-stochastic $|X|\times|X|$ matrices, therefore one can write $T(y\mid x) \equiv T_{xy}$ as the $xy-$th entry of the matrix $T$. Then, the representation of a constant policy as an $X\times U$ matrix can be written as $\overline{\pi}=\mathbf{1}_{|X|}v^\top$, where $\mathbf{1}_{|X|}$ where $v\in \Delta(U)$ is any probability distribution over the action space. Observe that, applying the operator $\langle \pi,T\rangle$ to a constant policy yields $\langle \overline{\pi},T\rangle=T\mathbf{1}_{|X|}v^\top$.
By the Perron-Frobenius Theorem \citep{horn2012matrix}, since $T$ is row-stochastic it has at least one eigenvalue $\operatorname{eig}(T)=1$, and this admits a (strictly positive) eigenvector $T\mathbf{1}_{|X|}=\mathbf{1}_{|X|}$. Therefore, 
$\langle \overline{\pi},T\rangle=T\mathbf{1}_{|X|}v^\top=\mathbf{1}_{|X|}v^\top = \overline{\pi}\implies \overline{\Pi}\subseteq \Pi_T.$ Combining this result with Proposition \ref{prop:InclusionDisadvantageMaxRobust}, we simply need to show that $\Pi_T\subset \Pi_D$. Take $\pi$ to be a fixed point of $\langle \pi,T\rangle$. Then $\langle \pi,T\rangle=\pi$, and from the definition in \eqref{eq:valdis}:
\begin{equation*}\begin{aligned}
D^\pi(x,T)
= V^{\pi}(x)-\mathbb{E}_{\substack{u \sim \langle \pi, T \rangle (x,\cdot)}} [ Q^\pi(x,u)]= V^{\pi}(x)-\mathbb{E}_{\substack{u \sim \pi (x,\cdot)}} [ Q^\pi(x,u)]=0.
\end{aligned}
\end{equation*}
Therefore, $\pi\in\Pi_D$, which completes the sequence of inclusions. Convexity of $\overline{\Pi}, \Pi_{T}$ follows from considering the convex hulls of two constant or fixed point policies.
\end{proof}

Let us reflect on the inclusion relations of Theorem \ref{theo:InclusionTheorem}. The inclusions are in general not strict, and in fact the geometry of the sets (as well as whether some of the relations are in fact equalities) is highly dependent on the reward function, and in particular on the complexity (from an information-theoretic perspective) of the reward function. As an intuition, less complex reward functions (more uniform) will make the inclusions above expand to the entire policy set, and more complex reward functions will make the relations collapse to equalities.

\begin{corollary}\label{cor:1}
For any \emph{ergodic} DOMDP there exist reward functions $\overline{R}$ and $\underline{R}$ such that the resulting DOMDP satisfies A) $\Pi_D = \Pi_0 = \Pi$ (any policy is max. robust) if $R = \overline{R}$, and B) $\Pi_T=\Pi_D = \Pi_0$ (only fixed point policies are maximally robust) if $R = \underline{R}$.
\end{corollary}
\begin{proof}[Corollary \ref{cor:1}]
    For statement A) let $\overline{R}(\cdot,\cdot,\cdot)=c$ for some constant $c\in\mathbb{R}$. Then, $J(\pi)=\mathbb{E}_{x_0\sim \mu_0}[\sum_{t} \gamma^t \overline{r}_t \mid \pi]= \frac{c \gamma}{1-\gamma}$, which does not depend on the policy $\pi$. For any noise kernel $T$ and policy $\pi$, $J(\pi)-J\langle \pi,T\rangle=0 \implies \pi\in \Pi_0$. For statement B assume $\exists \pi\in\Pi_0:\pi\notin\Pi_{T}$. Then, $\exists x^* \in X$ and $u^* \in U$ such that $\pi(x^*, u^*) \neq \langle \pi,T\rangle(x^*, u^*)$. Let 
    $\underline{R}(x,u,x') := c $ if $x = x^* $ and $ u = u^*$, 0 otherwise. Then, $\mathbb{E}[R(x,\pi(x),x']<\mathbb{E}[R(x,\langle \pi,T\rangle(x),x']$ and since the MDP is ergodic $x$ is visited infinitely often and $J(\pi)-J(\langle \pi,T\rangle) >0 \implies \pi\notin \Pi_0$, which contradicts the assumption. Therefore, $\Pi_0\setminus \Pi_{T}=\emptyset \implies \Pi_0=\Pi_T$.
    \end{proof}

We can now summarise the insights from Theorem \ref{theo:InclusionTheorem} and Corollary \ref{cor:1} in the following conclusions: (1) The set $\overline{\Pi}$ is maximally robust, convex and \emph{independent of the DOMDP}, (2) The set $\Pi_{T}$ is maximally robust, convex, includes $\overline{\Pi}$, and its properties \emph{only depend} on $T$, (3) The set $\Pi_{D}$ includes $\Pi_{T}$ and is maximally robust, but its properties \emph{depend on the DOMDP}.

\section{Robustness through Lexicographic Objectives}
\label{sec:lexicographicOpt}
To be able to apply LRL results to our robustness problem we need to first cast robustness as a valid objective to be maximised, and then show that a stochastic gradient descent approach would indeed find a global maximum of the objective, therefore yielding a maximally robust policy. \footnote{The advantage of using LRL is that we can formally bound the trade-off between \emph{robustness and optimality} through $\epsilon$, determinining how far we allow our resulting policy to be from an optimal policy in favour of it being more robust.}
\begin{wrapfigure}{L}{0.45\textwidth}
\vspace{-20pt}
\begin{minipage}{0.45\textwidth}
\begin{algorithm}[H]
\caption{LRPG}\label{alg:cap}
\begin{algorithmic}
\STATE \textbf{input} Simulator, $\tilde{T}$, $\epsilon$
\STATE initialise $\theta$, critic (if using), $\lambda$, $\{\beta_t^1,\beta_t^2,\eta\}$
\STATE set $t=0$, $x_t\sim \mu_0$
\WHILE{$t<\operatorname{max\_iterations}$}
\STATE perform $u_t\sim \pi_{\theta}(x_t)$
\STATE observe $r_t$, $x_{t+1}$, sample $y\sim \tilde{T}(\cdot\mid x)$
\IF {$\hat{K}_1(\theta)$ not converged}
\STATE $\hat{k}_1\leftarrow \hat{K}_1(\theta)$
\ENDIF
\STATE update critic (if using)
\STATE update $\theta$ using \eqref{eq:lrpgiter} and $\lambda$ using \eqref{eq:LRLupdates}
\ENDWHILE
\STATE \textbf{output} $\theta$
\end{algorithmic}
\end{algorithm}
\end{minipage}
\vspace{-20pt}
\end{wrapfigure}
\paragraph{Proposed approach} Following the framework presented in previous sections, we propose the following approach to obtain lexicographic robustness. In the introduction, we emphasised that the motivation for this work comes partially from the fact that we may not know $T$ in reality, or have a way to estimate it. However, the theoretical results until now depend on $T$. Our proposed solution to this lies in the results of Theorem \ref{theo:InclusionTheorem}. We can use a \emph{design} generator $\tilde{T}$ to perturb the policy during training such that $\tilde{T}$ has the \emph{smallest possible fixed point set} (i.e. the constant policy set, $\tilde{T}$ satisfies $\Pi_{\tilde{T}}=\overline{\Pi}$), and any algorithm that drives the policy towards the set of fixed points of $\tilde{T}$ \emph{will also drive the policy towards fixed points of $T$}: from Theorem \ref{theo:InclusionTheorem}, $\Pi_{\tilde{T}}\subseteq \Pi_{T}$.
\subsection{Lexicographically Robust Policy Gradient}
Consider then the objective to be minimised:
\begin{equation}
    K_{\tilde{T}}(\theta) =  \frac{1}{2}\sum_{x\in X}\mu_{\pi_\theta}(x) \sum_{u\in U} \big(\pi_{\theta}(x,u)-\langle \pi_{\theta},\tilde{T}\rangle (x,u)\big)^2, 
    \label{eq:ProxyOptimisationPolicies}
\end{equation}
Notice that optimising \eqref{eq:ProxyOptimisationPolicies} projects the current policy onto the set of fixed points of the operator $\langle \cdot, \tilde{T} \rangle$, and due to Assumption \ref{as:2}, which requires $\mu_{\pi_\theta}(x) > 0$ for all $x \in X$, the optimal solution is equal to zero if and only if there exists a value of the parameter $\theta$ for which the corresponding $\pi_\theta$ is a fixed point of $\langle \cdot, \tilde{T} \rangle$. We present now the proposed LRPG meta-algorithm to achieve lexicographic robustness for any policy gradient algorithm at choice.
From \cite{Skalse2022}, the convergence of PB-LRL algorithms is guaranteed as long as the original policy gradient algorithm for each single objective converges.
\begin{assumption}\label{as:5}
    The policy is updated through an algorithm (e.g. A2C, PPO...) such that $\theta_{t+1} \leftarrow \operatorname{proj}_{\Theta}\big[\theta_{t} + \alpha_t\nabla_{\theta_t}\hat{K}_1\big]$ converges \emph{a.s.} to a (local or global) optimum $\theta^*$.
\end{assumption}
\begin{theorem}\label{the:2}
    Consider a DOMDP as in Definition \ref{def:DOMDP} and let $\pi_\theta$ be a parameterised policy. Take a design kernel $\tilde{T}\in\{T:\Pi_{T}=\overline{\Pi}\}$. Consider the following modified gradient for objective $K_{\tilde{T}}(\theta)(x)$ and sampled point $y\sim \tilde{T}(\cdot\mid x)$:
    \begin{equation}\begin{aligned}\label{eq:approxgrad}
        \nabla_{\theta}\hat{K}_{\tilde{T}}' (\theta)=\mathbb{E}_{x\sim\mu_{\pi_\theta}}\big[\sum_{u\in U}(\pi_{\theta}(x,u)-\pi_{\theta}(y,u))\nabla_{\theta}\pi_{\theta}(x,u)\big].
        \end{aligned}
    \end{equation}
    Given an $\epsilon > 0$, if Assumptions \ref{as:2} and \ref{as:5} hold, then the following iteration (LRPG):
    \begin{equation}\label{eq:lrpgiter}
        \theta \leftarrow \operatorname{proj}_{\Theta}\big[\theta + (\beta_t^1 + \lambda \beta_t^2) \cdot \nabla_{\theta}\hat{K}_1(\theta) + \beta_t^2 \nabla_{\theta} \hat{K}_{\tilde{T}}'(\theta)\big]
    \end{equation}
    converges \emph{a.s.} to parameters $\theta^\epsilon$ that satisfy $
    \theta^\epsilon \in \argmin_{\theta \in \Theta'} K_{\tilde{T}}(\theta)$ such that $K^*_1 \geq K_1(\theta^\epsilon) - \epsilon,$
    where $\Theta' = \Theta$ if $\theta^*$ is globally optimal and a compact local neighbourhood of $\theta^*$ otherwise.
\end{theorem}
    \begin{proof}
    To apply LRL results, we need to show that both gradient descent schemes converge (separately) to local or global maxima. Let us first show that $\theta_{t+1} = \operatorname{proj}_{\Theta}\big[\theta_t - \alpha_t \nabla_{\theta}\hat{K}_{\tilde{T}}' (\theta_t)\big]$ converges \emph{a.s.} to parameters $\tilde{\theta}$ satisfying $K_{\tilde{T}}=0$. We prove this
    making use of fixed point iterations with non-expansive operators (specifically, Theorem 4, section 10.3 in \citet{Borkar2008}). First, observe that for a tabular representation, $\pi_\theta(x,u) = \theta_{xu}$, and $\nabla_{\theta}\pi_{\theta}(x,u)$ is a vector of zeros, with value $1$ for the position $\theta_{xu}$. We can then write the SGD in terms of the policy for each state $x$, considering $\pi(x) \equiv (\theta_{xu_1},\theta_{xu_2},...,\theta_{xu_k})^T$. Let $y\sim \tilde{T}(\cdot\mid x)$. Then:
    \begin{equation*}
        \begin{aligned}
            \pi_{t+1}(x)
            = \pi_t(x) - \alpha_t \big(\pi_t(x) - \pi_t(y)\big)= \pi_t(x) - \alpha_t \Big(\pi_t(x)-\langle \pi_t,\tilde{T}\rangle(x) -\big( \pi_t(y)-\langle \pi_t,\tilde{T}\rangle(x) \big) \Big).
        \end{aligned}
    \end{equation*}
    We now need to verify that the necessary conditions for applying Theorem 4, section 10.3 in \citet{Borkar2008} hold. First, making use of the property $\|\tilde{T}\|_{\infty}=1$ for any row-stochastic matrix $\tilde{T}$, for any two policies $\pi_1,\pi_2\in \Pi$:
    \begin{equation*}\begin{aligned}
        \|\langle \pi_1,\tilde{T}\rangle-\langle \pi_2,\tilde{T}\rangle \|_{\infty}= \|\tilde{T}\pi_1- \tilde{T}\pi_2 \|_{\infty} =\|\tilde{T}(\pi_1- \pi_2) \|_{\infty}\leq \|\tilde{T}\|_{\infty}\|\pi_1- \pi_2 \|_{\infty}= \|\pi_1- \pi_2 \|_{\infty}.
        \end{aligned}
    \end{equation*}
    Therefore, the operator $\langle \cdot,\tilde{T}\rangle$ is non-expansive with respect to the sup-norm. For the final condition:
    \begin{equation*}\begin{aligned}
        \mathbb{E}_{y\sim \tilde{T}(\cdot\mid x)}\left[\pi_t(y)-\langle \pi_t,\tilde{T}\rangle(x)\mid \pi_t,\tilde{T}\right]=\sum_{y\in X}\tilde{T}(y\mid x)\pi_t(y) - \langle \pi_t,\tilde{T}\rangle(x)=0.
        \end{aligned}
    \end{equation*}
    Therefore, the difference $\pi_t(y)-\langle \pi_t,\tilde{T}\rangle(x)$ is a martingale difference for all $x$. One can then apply Theorem 4, sec. 10.3 \citep{Borkar2008} to conclude that $\pi_{t}(x)\to \tilde{\pi}(x)$ almost surely. Finally from Assumption \ref{as:2}, for any policy all states $x\in X$ are visited infinitely often, therefore $\pi_{t}(x)\to \tilde{\pi}(x) \forall x\in X\implies \pi_{t}\to \tilde{\pi} $ and $\tilde{\pi}$ satisfies $\langle \tilde{\pi},\tilde{T}\rangle = \tilde{\pi}$, and $K_{\tilde{T}}(\tilde{\pi})=0$.

    Now, from Assumption \ref{as:5}, the iteration $\theta \leftarrow \operatorname{proj}_{\Theta}\big[\theta + \alpha_t\nabla_{\theta}\hat{K}_1\big]$ converges \emph{a.s.} to a (local or global) optimum $\theta^*$. Then, both objectives are invex \cite{ben1986invexity} (either locally or globally), and any linear combination of them will also be invex (again, locally or globally). Finally, we can directly apply the results from \cite{Skalse2022}, and
    \begin{equation*}
        \theta \leftarrow \operatorname{proj}_{\Theta}\big[\theta + (\beta_t^1 + \lambda \beta_t^2) \cdot \nabla_{\theta}\hat{K}_1(\theta) + \beta_t^2 \nabla_{\theta} \hat{K}_{\tilde{T}}'(\theta)\big]
    \end{equation*}
    converges \emph{a.s.} to parameters $\theta^\epsilon$ that satisfy 
    $\theta^\epsilon \in \argmin_{\theta \in \Theta'} K_{\tilde{T}}(\theta)$ such that $K^*_1 \geq K_1(\theta^\epsilon) - \epsilon$,
    where $\Theta' = \Theta$ if $\theta^*$ is globally optimal and a compact local neighbourhood of $\theta^*$ otherwise.
\end{proof}

\begin{remark}
Observe that \eqref{eq:approxgrad} is not the true gradient of \eqref{eq:ProxyOptimisationPolicies}, and $\theta^\epsilon \in \argmin_{\theta \in \Theta'} K_{\tilde{T}}(\theta)$ if there exists a (local) minimum of $K_{\tilde{T}}$ in $\Theta^{\epsilon}:=\{\theta:K_1^*\geq K_1(\theta) - \epsilon \}$. However, from Theorem \ref{the:2} we know that the (pseudo) gradient descent scheme converges to a global minimum in the tabular case, therefore $\langle \nabla_{\theta}\hat{K}_{\tilde{T}}' (\theta),\nabla_{\theta}\hat{K}_{\tilde{T}} (\theta)\rangle< 0$ \citep{Borkar2008}, and gradient-like descent schemes will converge to (local or) global minimisers, which motivates the choice of this gradient approximation.
\end{remark}
We reflect again on Figure \ref{fig:lem1}. The main idea behind LRPG is that by formally expanding the set of acceptable policies with respect to $K_1$, we may find robust policies more effectively while guaranteeing a minimum performance in terms of expected rewards. This addresses directly the premise behind Problem \ref{prob:1}. In LRPG the first objective is still to minimise the distance $J^* - J(\pi)$ up to some tolerance. Then, from the policies that satisfy this constraint, we want to steer the learning algorithm towards a maximally robust policy, and we can do so without knowing $T$.
\section{Considerations on Noise Generators}\label{sec:choice}
A natural question emerging is how to choose $\tilde{T}$, and how the choice influences the resulting policy robustness towards any other true $T$. In general, for any arbitrary policy utility landscape in a given MDP, there is no way of bounding the distance of the resulting policies for two different noise kernels $T_1,T_2$. However, \emph{the optimality of the policy} remains bounded: Through LRPG guarantees we know that, for both cases, the utility of the resulting policy will be at most $\epsilon$ far from the optimal.
\begin{corollary}\label{lem:poldis}
Take $T$ to be any arbitrary noise kernel, and $\tilde{T}$ to satisfy $\tilde{T}\in\{T:\Pi_{T}=\overline{\Pi}\}$. Let $\pi$ be a policy resulting from a LRPG algorithm. Assume that $\min_{\pi'\in\Pi_{\tilde{T}}}D_{TV}(\pi\Vert\pi')= a$ for some $a<1$. Then, it holds for any $T$ that $\min_{\pi'\in\Pi_{{T}}}D_{TV}(\pi\Vert\pi')\leq a$.
\end{corollary}
\begin{proof}
    The proof follows by the inclusion results in Theorem \ref{theo:InclusionTheorem}. If $\Pi_{\tilde{T}}=\overline{\Pi}$, then $\Pi_{\tilde{T}}\subseteq \Pi_{T}$ for any other $T$. Then, the distance from $\pi$ to the set $\Pi_{T}$ is at most the distance to $\Pi_{\tilde{T}}$.
\end{proof}
That is, when using LRPG to obtain a robust policy $\pi$, the resulting policy is at most $a$ far from the set of fixed points (and therefore a maximally robust policy) with respect to the true $T$. This is the key argument behind our choices for $\tilde{T}$: A priori, the most sensible choice is a kernel that has no other fixed point than the set of constant policies. This fixed point condition is satisfied in the discrete state case for any $\tilde{T}$ that induces an irreducible Markov Chain, and in continuous state for any $\tilde{T}$ that satisfies a reachability condition (\emph{i.e.} for any $x_0\in X$, there exists a finite time for which the probability of reaching any ball $B\subset X$ of radius $r>0$ through a sequence $x_{t+1}=T(x_t)$ is measurable). This holds for (additive) uniform or Gaussian disturbances.

\section{Experiments}
We verify the theoretical results of LRPG in a series of experiments on discrete state/action safety-related environments \citep{gym_minigrid} (for extended experiments in continuous control tasks, hyperparameters \emph{etc.} see \href{https://arxiv.org/abs/2209.15320}{extended version}). We use A2C \citep{sutton2018reinforcement} (LR-A2C) and PPO \citep{schulman2017proximal} (LR-PPO) for our implementations of LRPG. In all cases, the lexicographic tolerance was set to $\epsilon=0.99\hat{k}_1$ to deviate as little as possible from the primary objective. We compare against the baseline algorithms and against SA-PPO \citep{zhang2020robust} which is among the most effective (adversarial) robust RL approaches in literature. We trained 10 independent agents for each algorithm, and reported scores of the median agent (as in \citet{zhang2020robust}) for 50 roll-outs. To simulate $\tilde{T}$ we disturb $x$ as $\tilde{x}=x+\xi$ for (1) a uniform bounded noise signal $\xi\sim\mathcal{U}_{[-b,b]}$ ($\tilde{T}^u$) and (2) and a Gaussian noise ($\tilde{T}^g$) such that $\xi\sim\mathcal{N}(0,0.5)$. We test the resulting policies against a noiseless environment ($\emptyset$), a kernel $T_1 = \tilde{T}^u$, a kernel $T_2 = \tilde{T}^g$ and against two different state-adversarial noise configurations ($T^2_{adv}$) as proposed by \citet{zhang2021robust} to evaluate how effective LRPG is at rejecting adversarial disturbances.

\paragraph{Robustness Results} 
We use objectives as defined in \eqref{eq:ProxyOptimisationPolicies}. Additionally, we aim to test the hypothesis: If we have an estimator for the critic $Q^{\pi}$ we can obtain robustness without inducing regularity in the policy using $D^{\pi}$, yielding a larger policy subspace to steer towards, and hopefully achieving policies closer to optimal. For this, we consider the objective $K_{D}(\theta)(x):= \frac{1}{2}\|D^{\pi_{\theta}}(x,T)\|_2^2$ by modifying A2C to retain a Q critic. We investigate the impact of LRPG PPO and A2C for discrete action-space problems on Gymnasium \citep{brockman2016openai}. \emph{Minigrid-LavaGap} (fully observable), \emph{Minigrid-LavaCrossing} (partially observable) are safe exploration tasks where the agent needs to navigate an environment with cliff-like regions. \emph{Minigrid-DynamicObstacles} (stochastic, partially observable) is a dynamic obstacle-avoidance environment. See Table \ref{tab:1}.
\begin{table*}\centering
\tiny
\tabcolsep=0.11cm
\begin{tabular}{@{}lrrrr||rrrr@{}}\toprule
& \multicolumn{4}{c}{\emph{PPO on MiniGrid Environments}} & \multicolumn{4}{c}{\emph{A2C on MiniGrid Environments}} 
\\ \cmidrule{1-9}
Noise & PPO & LR$_{\text{PPO}}$$(K^u_{T})$ & LR$_{\text{PPO}}$$(K^g_{T})$ & SA-PPO & A2C & LR$_{\text{A2C}}$$(K^u_{T})$ & LR$_{\text{A2C}}$$(K^g_{T})$ & LR$_{\text{A2C}}$$(K_D)$\\ \midrule
\multicolumn{2}{l}{\emph{LavaGap}}\\
$\emptyset$& \bf{0.95}$\pm$0.003 &\bf{0.95}$\pm$0.075 & \bf{0.95}$\pm$0.101 & 0.94$\pm$0.068 & \bf{0.94}$\pm$0.004 & \bf{0.94}$\pm$0.005 & \bf{0.94}$\pm$0.003 & \bf{0.94}$\pm$0.006\\
$T_1$&0.80$\pm$0.041&\bf{0.95}$\pm$0.078&0.93$\pm$0.124&0.88$\pm$0.064&0.83$\pm$0.061&\bf{0.93}$\pm$0.019&0.89$\pm$0.032&0.91$\pm$0.088\\
$T_2$&0.92$\pm$0.015&\bf{0.95}$\pm$0.052&\bf{0.95}$\pm$0.094&0.93$\pm$0.050&0.89$\pm$0.029&\bf{0.94}$\pm$0.008&0.93$\pm$0.011&0.93$\pm$0.021\\
$T_{adv}^{2}$&0.01$\pm$0.051&{0.71}$\pm$0.251&{0.21}$\pm$0.357&\bf{0.87}$\pm$0.116&0.27$\pm$0.119&\bf{0.79}$\pm$0.069&0.68$\pm$0.127&0.56$\pm$0.249\\\midrule

\multicolumn{2}{l}{\emph{LavaCrossing}}\\
$\emptyset$&\bf{0.95}$\pm$0.023&{0.93}$\pm$0.050&{0.93}$\pm$0.018&0.88$\pm$0.091&{0.91}$\pm$0.024&{0.91}$\pm$0.063&{0.90}$\pm$0.017&\bf{0.92}$\pm$0.034\\
$T_1$&0.50$\pm$0.110&\bf{0.92}$\pm$0.053&{0.89}$\pm$0.029&0.64$\pm$0.109&0.66$\pm$0.071&\bf{0.78}$\pm$0.111&0.72$\pm$0.073&0.76$\pm$0.098\\
$T_2$&0.84$\pm$0.061&\bf{0.92}$\pm$0.050&\bf{0.92}$\pm$0.021&0.85$\pm$0.094&0.78$\pm$0.054&{0.83}$\pm$0.105&{0.86}$\pm$0.029&\bf{0.87}$\pm$0.063\\
$T_{adv}^{2}$&0.0$\pm$0.004&{0.50}$\pm$0.171&{0.38}$\pm$0.020&\bf{0.82}$\pm$0.072&0.06$\pm$0.056&{0.04}$\pm$0.030&0.01$\pm$0.008&\bf{0.09}$\pm$0.060\\\midrule

\multicolumn{2}{l}{\emph{DynamicObstacles}}\\
$\emptyset$&\bf{0.91}$\pm$0.002&\bf{0.91}$\pm$0.008&\bf{0.91}$\pm$0.007&\bf{0.91}$\pm$0.131&\bf{0.91}$\pm$0.011&0.88$\pm$0.020&0.89$\pm$0.009&\bf{0.91}$\pm$0.013\\
$T_1$&0.23$\pm$0.201&\bf{0.77}$\pm$0.102&0.61$\pm$0.119&0.45$\pm$0.188&0.27$\pm$0.104&0.43$\pm$0.108&0.45$\pm$0.162&\bf{0.56}$\pm$0.270\\
$T_2$&0.50$\pm$0.117&\bf{0.75}$\pm$0.075&0.70$\pm$0.072&0.68$\pm$0.490&0.45$\pm$0.086&0.53$\pm$0.109&{0.52}$\pm$0.161&\bf{0.67}$\pm$0.203\\
$T_{adv}^{2}$&-0.49$\pm$0.312&{0.51}$\pm$0.234&{0.33}$\pm$0.202&\bf{0.55}$\pm$0.170&-0.54$\pm$0.209&-{0.21}$\pm$0.192&-0.53$\pm$0.261&\bf{-0.51}$\pm$0.260\\
\bottomrule
\end{tabular}
\caption{Reward values gained by LRPG and baselines on discrete control tasks.}
\label{tab:1}
\vspace{-10pt}
\end{table*}

\section{Discussion}
\paragraph{Experiments} We applied LRPG on PPO and A2C (and SAC algorithms), for a set of discrete and continuous control environments. These environments are particularly sensitive to robustness problems; the rewards are sparse, and applying a sub-optimal action at any step of the trajectory often leads to terminal states with zero (or negative) reward. LRPG successfully induces lower robustness regrets in the tested scenarios, and the use of $K_D$ as an objective (even though we did not prove the convergence of a gradient based method with such objective) yields a better compromise between robustness and rewards. When compared to recent observational robustness methods, LRPG obtains similar robustness results while \emph{preserving the original guarantees of the chosen algorithm}.

\paragraph{Shortcomings and Contributions} The motivation for LRPG comes from situations where, when deploying a model-free controller in a dynamical system, we do not have a way of estimating the noise generation and we \emph{are required to retain convergence guarantees of the algorithms used}. Although LRPG is a useful approach for learning policies in control problems where the noise sources are unknown, questions emerge on whether there are more effective methods of incorporating robustness into RL policies when guarantees are not needed. Specifically, since a completely model-free approach does not allow for simple alternative solutions such as filtering or disturbance rejection, there are reasons to believe it could be outperformed by model-based (or model learning) approaches. However, we argue that in completely model-free settings, LRPG provides a rational strategy to robustify RL agents.

\bibliography{biblio_lrl}

\begin{thebibliography}{49}
\providecommand{\natexlab}[1]{#1}
\providecommand{\url}[1]{\texttt{#1}}
\expandafter\ifx\csname urlstyle\endcsname\relax
  \providecommand{\doi}[1]{doi: #1}\else
  \providecommand{\doi}{doi: \begingroup \urlstyle{rm}\Url}\fi

\bibitem[Abdullah et~al.(2019)Abdullah, Ren, Ammar, Milenkovic, Luo, Zhang, and
  Wang]{abdullah2019wasserstein}
Mohammed~Amin Abdullah, Hang Ren, Haitham~Bou Ammar, Vladimir Milenkovic, Rui
  Luo, Mingtian Zhang, and Jun Wang.
\newblock Wasserstein robust reinforcement learning.
\newblock \emph{arXiv preprint arXiv:1907.13196}, 2019.

\bibitem[Agarwal et~al.(2021)Agarwal, Kakade, Lee, and
  Mahajan]{agarwal2021theory}
Alekh Agarwal, Sham~M Kakade, Jason~D Lee, and Gaurav Mahajan.
\newblock On the theory of policy gradient methods: Optimality, approximation,
  and distribution shift.
\newblock \emph{J. Mach. Learn. Res.}, 22\penalty0 (98):\penalty0 1--76, 2021.

\bibitem[Behzadan and Munir(2017)]{behzadan2017vulnerability}
Vahid Behzadan and Arslan Munir.
\newblock Vulnerability of deep reinforcement learning to policy induction
  attacks.
\newblock In \emph{International Conference on Machine Learning and Data Mining
  in Pattern Recognition}, pages 262--275. Springer, 2017.

\bibitem[Ben-Israel and Mond(1986{\natexlab{a}})]{BenIsrael1986}
A.~Ben-Israel and B.~Mond.
\newblock What is invexity?
\newblock \emph{The Journal of the Australian Mathematical Society. Series B.
  Applied Mathematics}, 28\penalty0 (1):\penalty0 1--9, 1986{\natexlab{a}}.

\bibitem[Ben-Israel and Mond(1986{\natexlab{b}})]{ben1986invexity}
Adi Ben-Israel and Bertram Mond.
\newblock What is invexity?
\newblock \emph{The ANZIAM Journal}, 28\penalty0 (1):\penalty0 1--9,
  1986{\natexlab{b}}.

\bibitem[Bertsekas(1999)]{Bertsekas1999}
Dimitri Bertsekas.
\newblock \emph{Nonlinear Programming}.
\newblock Athena Scientific, 1999.

\bibitem[Bertsekas(1997)]{bertsekas1997nonlinear}
Dimitri~P Bertsekas.
\newblock Nonlinear programming.
\newblock \emph{Journal of the Operational Research Society}, 48\penalty0
  (3):\penalty0 334--334, 1997.

\bibitem[Bhandari and Russo(2019)]{bhandari2019global}
Jalaj Bhandari and Daniel Russo.
\newblock Global optimality guarantees for policy gradient methods.
\newblock \emph{arXiv preprint arXiv:1906.01786}, 2019.

\bibitem[Borkar(2008)]{Borkar2008}
Vivek~S. Borkar.
\newblock \emph{Stochastic Approximation}.
\newblock Hindustan Book Agency, 2008.

\bibitem[Borkar and Soumyanatha(1997)]{563625}
V.S. Borkar and K.~Soumyanatha.
\newblock An analog scheme for fixed point computation. i. theory.
\newblock \emph{IEEE Transactions on Circuits and Systems I: Fundamental Theory
  and Applications}, 44\penalty0 (4):\penalty0 351--355, 1997.
\newblock \doi{10.1109/81.563625}.

\bibitem[Brockman et~al.(2016)Brockman, Cheung, Pettersson, Schneider,
  Schulman, Tang, and Zaremba]{brockman2016openai}
Greg Brockman, Vicki Cheung, Ludwig Pettersson, Jonas Schneider, John Schulman,
  Jie Tang, and Wojciech Zaremba.
\newblock Openai gym.
\newblock \emph{arXiv preprint arXiv:1606.01540}, 2016.

\bibitem[Chevalier-Boisvert et~al.(2018)Chevalier-Boisvert, Willems, and
  Pal]{gym_minigrid}
Maxime Chevalier-Boisvert, Lucas Willems, and Suman Pal.
\newblock Minimalistic gridworld environment for openai gym.
\newblock \url{https://github.com/maximecb/gym-minigrid}, 2018.

\bibitem[Fu et~al.(2018)Fu, Luo, and Levine]{fu2018learning}
Justin Fu, Katie Luo, and Sergey Levine.
\newblock Learning robust rewards with adverserial inverse reinforcement
  learning.
\newblock In \emph{International Conference on Learning Representations}, 2018.

\bibitem[Gleave et~al.(2020)Gleave, Dennis, Wild, Kant, Levine, and
  Russell]{Gleave2020Adversarial}
Adam Gleave, Michael Dennis, Cody Wild, Neel Kant, Sergey Levine, and Stuart
  Russell.
\newblock Adversarial policies: Attacking deep reinforcement learning.
\newblock In \emph{International Conference on Learning Representations}, 2020.

\bibitem[Hanson(1981)]{Hanson1981}
Morgan~A Hanson.
\newblock On sufficiency of the kuhn-tucker conditions.
\newblock \emph{Journal of Mathematical Analysis and Applications}, 80\penalty0
  (2):\penalty0 545--550, 1981.

\bibitem[Heger(1994)]{heger1994consideration}
Matthias Heger.
\newblock Consideration of risk in reinforcement learning.
\newblock In \emph{Machine Learning Proceedings 1994}, pages 105--111.
  Elsevier, 1994.

\bibitem[Horn and Johnson(2012)]{horn2012matrix}
Roger~A Horn and Charles~R Johnson.
\newblock \emph{Matrix analysis}.
\newblock Cambridge university press, 2012.

\bibitem[Huang et~al.(2017)Huang, Papernot, Goodfellow, Duan, and
  Abbeel]{huang2017adversarial}
Sandy Huang, Nicolas Papernot, Ian Goodfellow, Yan Duan, and Pieter Abbeel.
\newblock Adversarial attacks on neural network policies.
\newblock \emph{arXiv preprint arXiv:1702.02284}, 2017.

\bibitem[Isermann(1982)]{isermann1982linear}
H~Isermann.
\newblock Linear lexicographic optimization.
\newblock \emph{Operations-Research-Spektrum}, 4\penalty0 (4):\penalty0
  223--228, 1982.

\bibitem[Klima et~al.(2019)Klima, Bloembergen, Kaisers, and
  Tuyls]{10.5555/3306127.3331713}
Richard Klima, Daan Bloembergen, Michael Kaisers, and Karl Tuyls.
\newblock Robust temporal difference learning for critical domains.
\newblock In \emph{Proceedings of the 18th International Conference on
  Autonomous Agents and MultiAgent Systems}, AAMAS '19, page 350–358,
  Richland, SC, 2019. International Foundation for Autonomous Agents and
  Multiagent Systems.
\newblock ISBN 9781450363099.

\bibitem[Korkmaz(2023)]{korkmaz2023adversarial}
Ezgi Korkmaz.
\newblock Adversarial robust deep reinforcement learning requires redefining
  robustness.
\newblock \emph{arXiv preprint arXiv:2301.07487}, 2023.

\bibitem[Kos and Song(2017)]{kos2017delving}
Jernej Kos and Dawn Song.
\newblock Delving into adversarial attacks on deep policies.
\newblock \emph{arXiv preprint arXiv:1705.06452}, 2017.

\bibitem[Liang et~al.(2022)Liang, Sun, Zheng, and Huang]{liang2022efficient}
Yongyuan Liang, Yanchao Sun, Ruijie Zheng, and Furong Huang.
\newblock Efficient adversarial training without attacking: Worst-case-aware
  robust reinforcement learning.
\newblock \emph{Advances in Neural Information Processing Systems},
  35:\penalty0 22547--22561, 2022.

\bibitem[Lin et~al.(2017)Lin, Hong, Liao, Shih, Liu, and Sun]{lin2017tactics}
Yen-Chen Lin, Zhang-Wei Hong, Yuan-Hong Liao, Meng-Li Shih, Ming-Yu Liu, and
  Min Sun.
\newblock Tactics of adversarial attack on deep reinforcement learning agents.
\newblock In \emph{Proceedings of the 26th International Joint Conference on
  Artificial Intelligence}, pages 3756--3762, 2017.

\bibitem[Mandlekar et~al.(2017)Mandlekar, Zhu, Garg, Fei-Fei, and
  Savarese]{mandlekar2017adversarially}
Ajay Mandlekar, Yuke Zhu, Animesh Garg, Li~Fei-Fei, and Silvio Savarese.
\newblock Adversarially robust policy learning: Active construction of
  physically-plausible perturbations.
\newblock In \emph{2017 IEEE/RSJ International Conference on Intelligent Robots
  and Systems (IROS)}, pages 3932--3939. IEEE, 2017.

\bibitem[Moos et~al.(2022)Moos, Hansel, Abdulsamad, Stark, Clever, and
  Peters]{moos2022robust}
Janosch Moos, Kay Hansel, Hany Abdulsamad, Svenja Stark, Debora Clever, and Jan
  Peters.
\newblock Robust reinforcement learning: A review of foundations and recent
  advances.
\newblock \emph{Machine Learning and Knowledge Extraction}, 4\penalty0
  (1):\penalty0 276--315, 2022.

\bibitem[Ng et~al.(1999)Ng, Harada, and Russell]{ng1999policy}
Andrew~Y Ng, Daishi Harada, and Stuart Russell.
\newblock Policy invariance under reward transformations: Theory and
  application to reward shaping.
\newblock In \emph{Proc. of the Sixteenth International Conference on Machine
  Learning, 1999}, 1999.

\bibitem[Pan et~al.(2019)Pan, Seita, Gao, and Canny]{pan2019risk}
Xinlei Pan, Daniel Seita, Yang Gao, and John Canny.
\newblock Risk averse robust adversarial reinforcement learning.
\newblock In \emph{2019 International Conference on Robotics and Automation
  (ICRA)}, pages 8522--8528. IEEE, 2019.

\bibitem[Paternain et~al.(2019)Paternain, Chamon, Calvo{-}Fullana, and
  Ribeiro]{Paternain2019}
Santiago Paternain, Luiz F.~O. Chamon, Miguel Calvo{-}Fullana, and Alejandro
  Ribeiro.
\newblock Constrained reinforcement learning has zero duality gap.
\newblock In \emph{Proceedings of the 33rd International Conference on Neural
  Information Processing Systems}, pages 7553--7563, 2019.

\bibitem[Pattanaik et~al.(2018)Pattanaik, Tang, Liu, Bommannan, and
  Chowdhary]{10.5555/3237383.3238064}
Anay Pattanaik, Zhenyi Tang, Shuijing Liu, Gautham Bommannan, and Girish
  Chowdhary.
\newblock Robust deep reinforcement learning with adversarial attacks.
\newblock In \emph{Proceedings of the 17th International Conference on
  Autonomous Agents and MultiAgent Systems}, AAMAS '18, page 2040–2042,
  Richland, SC, 2018. International Foundation for Autonomous Agents and
  Multiagent Systems.

\bibitem[Pirotta et~al.(2013)Pirotta, Restelli, Pecorino, and
  Calandriello]{pirotta2013safe}
Matteo Pirotta, Marcello Restelli, Alessio Pecorino, and Daniele Calandriello.
\newblock Safe policy iteration.
\newblock In \emph{International Conference on Machine Learning}, pages
  307--315. PMLR, 2013.

\bibitem[Raffin et~al.(2021)Raffin, Hill, Gleave, Kanervisto, Ernestus, and
  Dormann]{stable-baselines3}
Antonin Raffin, Ashley Hill, Adam Gleave, Anssi Kanervisto, Maximilian
  Ernestus, and Noah Dormann.
\newblock Stable-baselines3: Reliable reinforcement learning implementations.
\newblock \emph{Journal of Machine Learning Research}, 22\penalty0
  (268):\penalty0 1--8, 2021.
\newblock URL \url{http://jmlr.org/papers/v22/20-1364.html}.

\bibitem[Rentmeesters et~al.(1996)Rentmeesters, Tsai, and
  Lin]{rentmeesters1996theory}
Mark~J Rentmeesters, Wei~K Tsai, and Kwei-Jay Lin.
\newblock A theory of lexicographic multi-criteria optimization.
\newblock In \emph{Proceedings of ICECCS'96: 2nd IEEE International Conference
  on Engineering of Complex Computer Systems (held jointly with 6th CSESAW and
  4th IEEE RTAW)}, pages 76--79. IEEE, 1996.

\bibitem[Schulman et~al.(2017)Schulman, Wolski, Dhariwal, Radford, and
  Klimov]{schulman2017proximal}
John Schulman, Filip Wolski, Prafulla Dhariwal, Alec Radford, and Oleg Klimov.
\newblock Proximal policy optimization algorithms.
\newblock \emph{arXiv preprint arXiv:1707.06347}, 2017.

\bibitem[Skalse et~al.(2022{\natexlab{a}})Skalse, Farrugia-Roberts, Russell,
  Abate, and Gleave]{skalse2022invariance}
Joar Skalse, Matthew Farrugia-Roberts, Stuart Russell, Alessandro Abate, and
  Adam Gleave.
\newblock Invariance in policy optimisation and partial identifiability in
  reward learning.
\newblock \emph{arXiv preprint arXiv:2203.07475}, 2022{\natexlab{a}}.

\bibitem[Skalse et~al.(2022{\natexlab{b}})Skalse, Hammond, Griffin, and
  Abate]{Skalse2022}
Joar Skalse, Lewis Hammond, Charlie Griffin, and Alessandro Abate.
\newblock Lexicographic multi-objective reinforcement learning.
\newblock In \emph{Proceedings of the Thirty-First International Joint
  Conference on Artificial Intelligence}, pages 3430--3436, jul
  2022{\natexlab{b}}.
\newblock \doi{10.24963/ijcai.2022/476}.

\bibitem[Slater(1950)]{Slater1950}
Morton Slater.
\newblock Lagrange multipliers revisited.
\newblock Cowles Commission Discussion Paper No. 403, 1950.

\bibitem[Spaan(2012)]{spaan2012partially}
Matthijs~TJ Spaan.
\newblock Partially observable markov decision processes.
\newblock In \emph{Reinforcement Learning}, pages 387--414. Springer, 2012.

\bibitem[Spaan and Vlassis(2004)]{spaan2004point}
Matthijs~TJ Spaan and N~Vlassis.
\newblock A point-based pomdp algorithm for robot planning.
\newblock In \emph{IEEE International Conference on Robotics and Automation,
  2004. Proceedings. ICRA'04. 2004}, volume~3, pages 2399--2404. IEEE, 2004.

\bibitem[Sutton and Barto(2018)]{sutton2018reinforcement}
Richard~S Sutton and Andrew~G Barto.
\newblock \emph{Reinforcement learning: An introduction}.
\newblock MIT press, 2018.

\bibitem[Tan et~al.(2020)Tan, Esfandiari, Lee, Sarkar,
  et~al.]{tan2020robustifying}
Kai~Liang Tan, Yasaman Esfandiari, Xian~Yeow Lee, Soumik Sarkar, et~al.
\newblock Robustifying reinforcement learning agents via action space
  adversarial training.
\newblock In \emph{2020 American control conference (ACC)}, pages 3959--3964.
  IEEE, 2020.

\bibitem[Tessler et~al.(2019)Tessler, Efroni, and Mannor]{tessler2019action}
Chen Tessler, Yonathan Efroni, and Shie Mannor.
\newblock Action robust reinforcement learning and applications in continuous
  control.
\newblock In \emph{International Conference on Machine Learning}, pages
  6215--6224. PMLR, 2019.

\bibitem[van~der Pol et~al.(2020)van~der Pol, Kipf, Oliehoek, and
  Welling]{10.5555/3398761.3398926}
Elise van~der Pol, Thomas Kipf, Frans~A. Oliehoek, and Max Welling.
\newblock Plannable approximations to mdp homomorphisms: Equivariance under
  actions.
\newblock In \emph{Proceedings of the 19th International Conference on
  Autonomous Agents and MultiAgent Systems}, AAMAS '20, page 1431–1439,
  Richland, SC, 2020. International Foundation for Autonomous Agents and
  Multiagent Systems.
\newblock ISBN 9781450375184.

\bibitem[Van~Parys et~al.(2015)Van~Parys, Kuhn, Goulart, and
  Morari]{van2015distributionally}
Bart~PG Van~Parys, Daniel Kuhn, Paul~J Goulart, and Manfred Morari.
\newblock Distributionally robust control of constrained stochastic systems.
\newblock \emph{IEEE Transactions on Automatic Control}, 61\penalty0
  (2):\penalty0 430--442, 2015.

\bibitem[Wiesemann et~al.(2014)Wiesemann, Kuhn, and
  Sim]{wiesemann2014distributionally}
Wolfram Wiesemann, Daniel Kuhn, and Melvyn Sim.
\newblock Distributionally robust convex optimization.
\newblock \emph{Operations Research}, 62\penalty0 (6):\penalty0 1358--1376,
  2014.

\bibitem[Xu and Mannor(2006)]{xu2006robustness}
Huan Xu and Shie Mannor.
\newblock The robustness-performance tradeoff in markov decision processes.
\newblock \emph{Advances in Neural Information Processing Systems}, 19, 2006.

\bibitem[Zhang et~al.(2020)Zhang, Chen, Xiao, Li, Liu, Boning, and
  Hsieh]{zhang2020robust}
Huan Zhang, Hongge Chen, Chaowei Xiao, Bo~Li, Mingyan Liu, Duane Boning, and
  Cho-Jui Hsieh.
\newblock Robust deep reinforcement learning against adversarial perturbations
  on state observations.
\newblock \emph{Advances in Neural Information Processing Systems},
  33:\penalty0 21024--21037, 2020.

\bibitem[Zhang et~al.(2021)Zhang, Chen, Boning, and Hsieh]{zhang2021robust}
Huan Zhang, Hongge Chen, Duane Boning, and Cho-Jui Hsieh.
\newblock Robust reinforcement learning on state observations with learned
  optimal adversary.
\newblock In \emph{International Conference on Learning Representation (ICLR)},
  2021.

\bibitem[Zhang(2018)]{deeprl}
Shangtong Zhang.
\newblock Modularized implementation of deep rl algorithms in pytorch.
\newblock \url{https://github.com/ShangtongZhang/DeepRL}, 2018.

\end{thebibliography}

\newpage
\appendix
\section{Examples and Further Considerations}\label{apx:ex}
We provide here two examples to show how we can obtain limit scenarios  $\Pi_0=\Pi$ (any policy is maximally robust) or $\Pi_0=\Pi_T$ (Example 1), and how for some MDPs the third inclusion in Theorem \ref{theo:InclusionTheorem} is strict (Example 2).
\paragraph{Example 1} Consider the simple MDP in Figure \ref{fig:mdp}. First, consider the reward function $R_1(x_1,\cdot,\cdot) = 10,\,R_1(x_2,\cdot,\cdot) = 0$. This produces a ``dummy'' MDP where all policies have the same reward sum. Then, $\forall T,\pi$, $V^{\langle \pi,T\rangle}=V^{\pi}$, and therefore we have $\Pi_D=\Pi_0=\Pi$.
\begin{figure}[H]
\centering
   \includegraphics{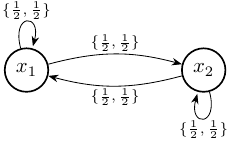}
     \caption{Example MDP. Values in brackets represent $\{P(\cdot,u_1,\cdot),P(\cdot,u_2,\cdot)\}$.}
     \label{fig:mdp}
\end{figure}
Now, consider the reward function $R_2(x_1,u_1,\cdot) = 10,\,R_2(\cdot,\cdot,\cdot) = 0$ elsewhere. Take a non-constant policy $\pi$, \emph{i.e.}, $\pi(x_1)\neq \pi(x_2)$. In the example DOMDP (assuming the initial state is drawn uniformly from $X_0=\{x_1,x_2\}$) one can show that at any time in the trajectory, there is a stationary probability $\Pr\{x_t=x_1\}=\frac{1}{2}$. Let us abuse notation and write $\pi(x_i)=(\begin{array}{cc}
\pi(x_i,u_1)& \pi(x_i,u_2)
\end{array})^\top$ and $R(x_i)=(\begin{array}{cc}
R(x_i,u_1,\cdot)& R(x_i,u_2,\cdot)
\end{array})^\top$. For the given reward structure we have $R(x_2)=(\begin{array}{cc}
0& 0
\end{array})^\top$, and therefore: 
\begin{equation}\label{eq:costex}
J(\pi) = E_{x_0\sim \mu_0}\left[\sum_{t=0}^{\infty}\gamma^t R_t\right]=\frac{1}{2}\langle R(x_1),\pi(x_1)\rangle \frac{\gamma}{1-\gamma}.
\end{equation}
Since the transitions of the MDP are independent of the actions, following the same principle as in \eqref{eq:costex}:
$J\langle \pi,T\rangle =\frac{1}{2}\langle R(x_1),\langle \cdot,T\rangle(\pi)(x_1)\rangle \frac{\gamma}{1-\gamma}.$
For any noise map $\langle \cdot,T\rangle\neq\operatorname{Id}$, for the two-state policy it holds that $\pi\notin\Pi_T\implies \langle \pi,T\rangle \neq \pi$. Therefore $\langle \pi,T\rangle(x_1)\neq\pi(x_1) $ and:
$$J(\pi)-J(\langle \pi,T\rangle) =\frac{5\gamma}{1-\gamma} \cdot \big( \pi(x_1,1)-\langle\pi,T\rangle(x_1,1) \big) \neq 0,$$
which implies that $\pi\notin\Pi_0$.
\paragraph{Example 2} Consider the same MDP in Figure \ref{fig:mdp} with reward function $R(x_1,u_1,\cdot) =R(x_2,u_1,\cdot)=10,$ and a reward of zero for all other transitions. Take a policy $\pi(x_1)=(1\,\,0)$, $\pi(x_2)=(0\,\,1)$. The policy yields a reward of $10$ in state $x_1$ and a reward of $0$ in state $x_2$. Again we assume the initial state is drawn uniformly from $X_0=\{x_1,x_2\}$. Then, observe: 
\begin{equation*}\begin{aligned}
J(\pi) =& E_{x_0\sim \mu_0}\left[\sum_{t=0}^{\infty}\gamma^t R_t\right]=\frac{1}{2}\langle R(x_1),\pi(x_1)\rangle \frac{\gamma}{1-\gamma}=\\
=&\frac{1}{2}\frac{10\gamma}{1-\gamma}= \frac{5\gamma}{1-\gamma}.
\end{aligned}
\end{equation*}
Define now noise map $T(\cdot\mid x_1)=(\frac{1}{2}\,\,\frac{1}{2})$ and $T(\cdot\mid x_2)=(\frac{1}{2}\,\,\frac{1}{2})$. Observe this noise map yields a policy with non-zero disadvantage, $D^{\pi}(x_1,T) = \frac{5\gamma}{1-\gamma}-\big(\frac{5\gamma}{1-\gamma}-2.5\big)=2.5$ and similarly $D^{\pi}(x_2,T) =-2.5$, therefore $\pi\notin\Pi_D$. However, the policy \emph{is maximally robust}:
\begin{equation}\label{eq:costex2}\begin{aligned}
J(\langle \pi,T\rangle) =\frac{1}{2}\langle R(x_1),\langle \pi,T\rangle(x_1)\rangle \frac{\gamma}{1-\gamma} +\\
+\frac{1}{2}\langle R(x_2),\langle \pi,T\rangle(x_2)\rangle \frac{\gamma}{1-\gamma}= \frac{1}{2}\frac{\gamma}{1-\gamma}\big(5 + 5\big)=\frac{5\gamma}{1-\gamma}.
\end{aligned}
\end{equation}
Therefore, $\pi\in\Pi_0$.
\section{Theoretical Results}\label{apx:theo}

\subsection{Auxiliary Results}

\begin{theorem}[Stochastic Approximation with Non-Expansive Operator]\label{th:approx}
    Let $\{\xi_t\}$ be a random sequence with $\xi_t\in\mathbb{R}^n$ defined by the iteration:
    \begin{equation*}
        \xi_{t+1} = \xi_{t} + \alpha_t(F(\xi_t)-\xi_t +M_{t+1}),
    \end{equation*}
    where:
    \begin{enumerate}
    \item The step sizes $\alpha_t$ satisfy standard learning rate assumptions.
        \item $F:\mathbb{R}^n\mapsto \mathbb{R}^n$ is a $\|\cdot\|_{\infty}$ non-expansive map. That is, for any $\xi_1,\xi_2\in\mathbb{R}^n$, $\|F(\xi_1)-F(\xi_2)\|_{\infty}\leq \|\xi_1-\xi_2\|_\infty$.
        \item $\{M_t\}$ is a martingale difference sequence with respect to the increasing family of $\sigma-$fields $\mathcal{F}_t:=\sigma(\xi_0,M_0,\xi_1,M_1,...,\xi_t,M_t)$.
    \end{enumerate}
Then, the sequence $\xi_t\to \xi^*$ \emph{almost surely} where $\xi^*$ is a fixed point such that $F(\xi^*) = \xi^*$. 
\end{theorem}
\begin{proof}
    See \citet{563625}.
\end{proof}
\begin{theorem}[PB-LRL Convergence with 2 objectives.\citep{Skalse2022}]\label{the:lrl}
Let $\mathcal{M}$ be a multi-objective MDP with objectives $K_i$, $i\in \{1,2\}$. Assume a policy $\pi$ is twice differentiable in parameters $\theta$, and if using a critic $V_i$ assume it is continuously differentiable on parameters $w_i$. Choose a tolerance $\epsilon$, and suppose that if PB-LRL is run for $T$ steps, there exists some limit point $w_i\to w_i^*(\theta)$ when $\theta$ is held fixed. If for both objectives there exists a gradient descent scheme such that $\lim_{T\to \infty}\mathbb{E}_t[\theta_t]\in \Theta_i^\epsilon$ then combining the objectives as in \eqref{eq:LRLupdates} yields $\lim_{T\to \infty}\mathbb{E}_t[\theta_t]\in \{\argmax_\theta K_2(\theta), 
\, \text{s.t.} \,\, 
K_1(\theta)\geq K_1(\theta')- \epsilon\}$.
\end{theorem}
\begin{proof}[Proof Sketch]
    We refer the interested reader to \citet{Skalse2022} for a full proof, and here attempt to provide the intuition behind the result in the form of a proof sketch. 
    
    Let us begin by briefly recalling the general problem statement: we wish to take a multi-objective MDP $\mathcal{M}$ with $m$ objectives, and obtain a lexicographically optimal policy (one that optimises the first objective, and then subject to this optimises the second objective, and so on). More precisely, for a policy $\pi$ parameterised by $\theta$, we say that $\pi$ is (globally) \textit{lexicographically $\epsilon$-optimal} if $\theta \in \Theta^\epsilon_m$, where $\Theta^\epsilon_0 = \Theta$ is the set of all policies in $\mathcal{M}$, $\Theta^\epsilon_{i+1} := \{\theta \in \Theta^\epsilon_i \mid \max_{\theta' \in \Theta^\epsilon_i} K_i(\theta') - K_i(\theta) \leq \epsilon_i \}$, and $\mathbb{R}^{m-1} \ni \epsilon \succcurlyeq 0$.\footnote{The proof in \citet{Skalse2022} also considers \emph{local} lexicographic optima, though for the sake of simplicity, we do not do so here.}
    
    The basic idea behind policy-based lexicographic reinforcement learning (PB-LRL) is to use a multi-timescale approach to first optimise $\theta$ using $K_1$, then at a slower timescale optimise $\theta$ using $K_2$ while adding the condition that the loss with respect to $K_1$ remains bounded by its current value, and so on. This sequence of constrained optimisations problems can be solved using a Lagrangian relaxation \citep{Bertsekas1999}, either in series or -- via a judicious choice of learning rates -- simultaneously, by exploiting a separation in timescales \citep{Borkar2008}. In the simultaneous case, the parameters of the critic $w_i$ (if using an actor-critic algorithm, if not this part of the argument may be safely ignored) for each objective are updated on the fastest timescale, then the parameters $\theta$, and finally (i.e., most slowly) the Lagrange multipliers for each of the remaining constraints.
    
    The proof proceeds via induction on the number of objectives, using a standard stochastic approximation argument \citep{Borkar2008}. In particular, due to the learning rates chosen, we may consider those more slowly updated parameters fixed for the purposes of analysing the convergence of the more quickly updated parameters. In the base case where $m=1$, we have (by assumption) that $\lim_{T\to \infty}\mathbb{E}_t[\theta]\in \Theta_1^\epsilon$. This is simply the standard (non-lexicographic) RL setting. Before continuing to the inductive step, \citet{Skalse2022} observe that because gradient descent on $K_1$ converges to globally optimal stationary point when $m = 1$ then $K_1$ must be globally \emph{invex} (where the opposite implication is also true) \citep{BenIsrael1986}.\footnote{A differentiable function $f : \mathbb{R}^n \rightarrow \mathbb{R}$ is (globally) invex if and only if there exists a function $g : \mathbb{R}^n \times \mathbb{R}^n \rightarrow \mathbb{R}^n$ such that $f(x_1) - f(x_2) \geq g(x_1,x_2)^\top \nabla f(x_2)$ for all $x_1, x_2 \in \mathbb{R}^n$ \citep{Hanson1981}.}
    
    The reason this observation is useful is that because each of the objectives $K_i$ shares the same functional form, they are all invex, and furthermore, invexity is conserved under linear combinations and the addition of scalars, meaning that the Lagrangian formed in the relaxation of each constrained optimisation problem is also invex. As a result, if we assume that $\lim_{T\to \infty}\mathbb{E}_t[\theta]\in \Theta_i^\epsilon$ as our inductive hypothesis, then the stationary point of the Lagrangian for optimising objective $K_{i+1}$ is a global optimum, given the constraints that it does not worsen performance on $K_1, \ldots, K_i$. Via Slater's condition \citep{Slater1950} and standard saddle-point arguments \citep{Bertsekas1999,Paternain2019}, we therefore have that $\lim_{T\to \infty}\mathbb{E}_t[\theta]\in \Theta_{i+1}^\epsilon$, completing the inductive step, and thus the overall inductive argument.
    
    This concludes the proof that $\lim_{T\to \infty}\mathbb{E}_t[\theta]\in \Theta_m^\epsilon$. We refer the reader to \citet{Skalse2022} for a discussion of the error $\epsilon$, but intuitively it corresponds to a combination of the representational power of $\theta$, the critic parameters $w_i$ (if used), and the duality gap due to the Lagrangian relaxation \citep{Paternain2019}. In cases where the representational power of the various parameters is sufficiently high, then it can be shown that $\epsilon = 0$.
\end{proof}

\subsection{On Adversarial Disturbances and other Noise Kernels}\label{apx:T}
A problem that remains open after this work is what constitutes an appropriate choice of $\tilde{T}$, and what can we expect by restricting a particular class of $\tilde{T}$. We first discuss adversarial examples, and then general considerations on $\tilde{T}$ versus $T$.
\paragraph{Adversarial Noise} As mentioned in the introduction, much of the previous work focuses on adversarial disturbances. We did not directly address this in the results of this work since our motivation lies in the scenarios where the disturbance is not adversarial and is unknown. However, following the results of Section \ref{sec:3}, we are able to reason about adversarial disturbances. Consider an adversarial map $T_{adv}$ to be
$$\langle \pi,T_{adv}\rangle (x) = \pi(y), \quad y\in \operatorname{argmax}_{y\in X_{ad}(x)}d\big(\pi(x),\pi(y)\big), $$
with $X_{ad}(x)\subseteq X$ being a set of admissible disturbance states for $x$, and $d(\cdot,\cdot)$ is a distance measure between distributions (\emph{e.g.} 2-norm).
\begin{proposition}
    Constant policies are a fixed point of $T_{adv}$, and are the only fixed points if for all pairs $x_0,x_k$ there exists a sequence $\{x_0,...,x_k\}\subseteq X$ such that $x_{i}\in X_{ad}(x_{i})$.
\end{proposition}
\begin{proof}
    First, it is straight-forward that if $\overline{\pi}\in \overline{\Pi}\implies \langle \overline{\pi},T_{adv}\rangle (x) = \overline{\pi}(x)$. To show they are the only fixed points, assume that there is a non-constant policy $\pi'$ that is a fixed point of $T_{ad}$. Then, there exists $x,z$ such that $\pi'(x)\neq \pi'(z)$. However, by assumption, we can construct a sequence $\{x,...,z\}\subseteq X$ that connects $x$ and $z$ and every state in the sequence is in the admissible set of the previous one. Assume without loss of generality that this sequence is $\{x,y,z\}$. Then, if $\pi'$ is a fixed point, $\langle {\pi'},T_{adv}\rangle (x)=\pi'(x)$, $\langle {\pi'},T_{adv}\rangle (y)=\pi'(y)$ and $\langle {\pi'},T_{adv}\rangle (z)=\pi'(z)$. However, $\pi'(x)\neq \pi'(z)$, so either $\pi'(x)\neq \pi'(y)\implies d(\pi'(x), \pi'(y))\neq 0$ or $\pi'(y)\neq \pi'(z)\implies d(\pi'(y), \pi'(z))\neq 0$, therefore $\pi'$ cannot be a fixed point of $T_{adv}$.
\end{proof}
The main difference between an adversarial operator and the random noise considered throughout this work is that $T_{adv}$ is \emph{not a linear operator}, and additionally, it is time varying (since the policy is being modified at every time step of the PG algorithm). Therefore, including it as a LRPG objective would invalidate the assumptions required for LRPG to retain formal guarantees of the original PG algorithm used, and it is not guaranteed that the resulting policy gradient algorithm would converge.

\section{Experiment Methodology}\label{apx:sims}
We use in the experiments well-tested implementations of A2C, PPO and SAC from Stable Baselines 3 \citep{stable-baselines3} to include the computation of the lexicographic parameters in \eqref{eq:LRLupdates}. All experiments were run on an Ubuntu 18.04 system, with a 16 core CPU and a graphic card Nvidia GeForce 3060.

\paragraph{LRPG Parameters.} The LRL parameters are initialised in all cases as $\beta_0^1=2,\,\beta_0^2=1,\,\lambda =0$ and $\eta = 0.001$. The LRL tolerance is set to $\epsilon_t=0.99\hat{k}_1$ to ensure we never deviate too much from the original objective, since the environments have very sparse rewards. We use a first order approximation to compute the LRL weights from the original 
{LMORL implementation}.

\subsection{Discrete Control}
The discrete control environments used can be seen in Figure \ref{fig:envs}.
\begin{figure*}[h]
\centering
\resizebox{0.3\linewidth}{!}{\includegraphics{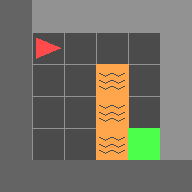}}
\resizebox{0.3\linewidth}{!}{\includegraphics{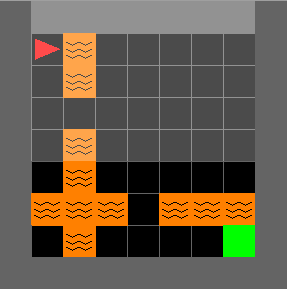}}
\resizebox{0.3\linewidth}{!}{\includegraphics{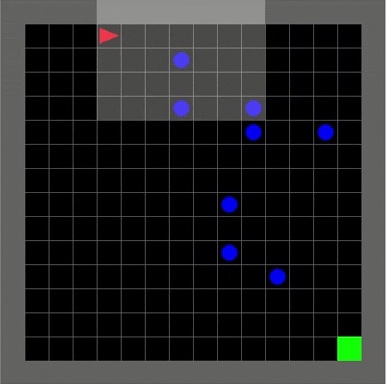}}
\caption{Screenshots of the environments used, from left: LavaGap, LavaCrossing and DynamicObstables.}
\label{fig:envs}
\end{figure*}
Since all the environments use a pixel representation of the observation, we use a shared representation for the value function and policy, where the first component is a convolutional network, implemented as in \citet{deeprl}. The hyper-parameters of the neural representations are presented in Table \ref{tab:2}.
\begin{table}[H]
\centering
\begin{tabular}{||c c c||} 
 \hline
Layer & Output & Func. \\ [0.5ex] 
 \hline\hline
 Conv1 & 16 & ReLu \\ 
 \hline
 Conv2 & 32 & ReLu \\
 \hline
 Conv3 & 64 & ReLu \\
 \hline
\end{tabular}
\caption{Shared Observation Layers}
\label{tab:2}
\end{table}
The actor and critic layers, for both algorithms, are a fully connected layer with $64$ features as input and the corresponding output. We used in all cases an Adam optimiser. We optimised the parameters for each (vanilla) algorithm through a quick parameter search, and apply the same parameters for the Lexicographically Robust versions.
\begin{table}[H]
\centering
\begin{tabular}{||l c c c||}
\hline
 & LavaGap & LavaCrossing & DynObs\\
 \hline\hline
 Parallel Envs & 16 & 16 & 16 \\
 \hline
 Steps & $2\cdot 10^6$ & $2\cdot 10^6$ & $8\times 10^6$ \\ 
 \hline
 $\gamma$ & 0.99 & 0.99 & 0.98 \\
 \hline
 $\alpha$ & 0.00176 & 0.00176 & 0.00181 \\
 \hline
 $\epsilon$(Adam) & $10^{-8}$ & $10^{-8}$ & $10^{-8}$  \\
 \hline
 Grad. Clip & 0.9 & 0.9 & 0.5 \\
 \hline
 Gae & 0.95 & 0.95 & 0.95 \\
 \hline
 Rollout & 64 & 64 & 64 \\
  \hline
 E. Coeff & 0.01 & 0.014 & 0.011 \\
  \hline
 V. Coeff & 0.05 & 0.05 & 0.88 \\
 \hline
\end{tabular}
\caption{A2C Parameters}
\label{tab:3}
\end{table}
\begin{table}[H]
\centering
\begin{tabular}{||l c c c||}
\hline
 & LavaGap & LavaCrossing & DynObs\\
 \hline\hline
 Parallel Envs & 8 & 8 & 8\\
 \hline
 Steps & $6\cdot 10^6$ & $2\cdot 10^6$ & $8\times 10^5$ \\ 
 \hline
 $\gamma$ & 0.95 & 0.99 & 0.97 \\
 \hline
 $\alpha$ & 0.001 & 0.001 & 0.001 \\
 \hline
 $\epsilon$(Adam) & $10^{-8}$ & $10^{-8}$ & $10^{-8}$ \\
 \hline
 Grad. Clip & 1 & 1 & 0.1 \\
 \hline
 Ratio Clip & 0.2 & 0.2 & 0.2\\
 \hline
 Gae & 0.95 & 0.95 & 0.95 \\
 \hline
 Rollout & 256 & 512 & 256 \\
  \hline
 Epochs & 10 & 10 & 10 \\
  \hline
 E. Coeff & 0 & 0.1 & 0.01 \\
 \hline
\end{tabular}
\caption{PPO Parameters}
\label{tab:6}
\end{table}
For the implementation of the LRPG versions of the algorithms, in all cases we allow the algorithm to iterate for $1/3$ of the total steps before starting to compute the robustness objectives. In other words, we use $\hat{K}(\theta)=K_1(\theta)$ until $t=\frac{1}{3}\operatorname{max\_steps}$, and from this point we resume the lexicographic robustness computation as described in Algorithm \ref{alg:cap}. This is due to the structure of the environments simulated. The rewards (and in particular the positive rewards) are very sparse in the environments considered. Therefore, when computing the policy gradient steps, the loss for the primary objective is practically zero until the environment is successfully solved at least once. If we implement the combined lexicographic loss from the first time step, many times the algorithm would converge to a (constant) policy without exploring for enough steps, leading to convergence towards a maximally robust policy that does not solve the environment. 

\textbf{Noise Kernels.} We consider two types of noise; a normal distributed noise $\tilde{T}^g$ and a uniform distributed noise $\tilde{T}^u$. For the environments LavaGap and DynamicObstacles, the kernel $\tilde{T}^u$ produces a disturbed state $\tilde{x}=x+\xi$ where $\|\xi\|_{\infty}\leq 2$, and for LavaCrossing $\|\xi\|_{\infty}\leq 1.5$. The normal distributed noise is in all cases $\mathcal{N}(0,0.5)$. The maximum norm of the noise is quite large, but this is due to the structure of the observations in these environments. The pixel values are encoded as integers $0-9$, where each integer represents a different feature in the environment (empty space, doors, lava, obstacle, goal...). Therefore, any noise $\|\xi\|_{\infty}\leq 0.5$ would most likely not be enough to \textit{confuse} the agent. On the other hand, too large noise signals are unrealistic and produce pathological environments. All the policies are then tested against two ``true" noise kernels, $T_1=\tilde{T}^u$ and $T_2=\tilde{T}^g$. The main reason for this is to test both the scenarios where we assume a \textit{wrong} noise kernel, and the case where we are training the agents with the correct kernel.

\textbf{Comparison with SA-PPO.} One of the baselines included is the State-Adversarial PPO algorithm proposed in \citet{zhang2020robust}. The implementation includes an extra parameter that multiplies the regularisation objective, $k_{ppo}$. Since we were not able to find indications on the best parameter for discrete action environments, we implemented $k_{ppo}\in \{0.1, 1, 2\}$ and picked the best result for each entry in Table \ref{tab:1}. Larger values seemed to de-stabilise the learning in some cases. The rest of the parameters are kept as in the vanilla PPO implementation.

\subsubsection{Extended Results: Adversarial Disturbances}

Even though we do not use an adversarial attacker or disturbance in our reasoning through this work, we implemented a policy-based state-adversarial noise disturbance to test the benchmark algorithms against, and evaluate how well each of the methods reacts to such adversarial disturbances.

\paragraph{Adversarial Disturbance} We implement a bounded policy-based adversarial attack, where at each state $x$ we maximise for the KL divergence between the disturbed and undisturbed state, such that the adversarial operator is:
\begin{equation*}\begin{aligned}
    T_{adv}^{\varepsilon}(y\mid x) = 1 \implies y\in\argmax_{\tilde{x}}&\,\,D_{KL}(\pi(x),\pi(\tilde{x})) \\
    s.t.& \,\,\|x-\tilde{x}\|_2\leq \varepsilon.
\end{aligned}
\end{equation*}
The optimisation problem is solved at every point by using a Stochastic Gradient Langevin Dynamics (SGLD) optimiser. The results are presented in Table \ref{tab:app}.
\begin{table*}\centering
\tiny
\tabcolsep=0.11cm
\begin{tabular}{@{}lrrrr||rrrr@{}}\toprule
& \multicolumn{4}{c}{\emph{PPO on MiniGrid Environments}} & \multicolumn{4}{c}{\emph{A2C on MiniGrid Environments}} 
\\ \cmidrule{1-9}
Noise & Vanilla & LR$_{\text{PPO}}$$(K^u_{T})$ & LR$_{\text{PPO}}$$(K^g_{T})$ & SA-PPO & Vanilla & LR$_{\text{A2C}}$$(K^u_{T})$ & LR$_{\text{A2C}}$$(K^g_{T})$ & LR$_{\text{A2C}}$$(K_D)$\\ \midrule
\multicolumn{2}{l}{\emph{LavaGap}}\\
$\emptyset$& \bf{0.95}$\pm$0.003 &\bf{0.95}$\pm$0.075 & \bf{0.95}$\pm$0.101 & 0.94$\pm$0.068 & \bf{0.94}$\pm$0.004 & \bf{0.94}$\pm$0.005 & \bf{0.94}$\pm$0.003 & \bf{0.94}$\pm$0.006\\
$T_1$&0.80$\pm$0.041&\bf{0.95}$\pm$0.078&0.93$\pm$0.124&0.88$\pm$0.064&0.83$\pm$0.061&\bf{0.93}$\pm$0.019&0.89$\pm$0.032&0.91$\pm$0.088\\
$T_2$&0.92$\pm$0.015&\bf{0.95}$\pm$0.052&\bf{0.95}$\pm$0.094&0.93$\pm$0.050&0.89$\pm$0.029&\bf{0.94}$\pm$0.008&0.93$\pm$0.011&0.93$\pm$0.021\\
$T_{adv}^{0.5}$&0.56$\pm$0.194&\bf{0.93}$\pm$0.101&0.91$\pm$0.076&0.90$\pm$0.123&0.92$\pm$0.034&\bf{0.94}$\pm$0.003&\bf{0.94}$\pm$0.007&0.93$\pm$0.015\\
$T_{adv}^{1}$&0.20$\pm$0.243&\bf{0.90}$\pm$0.124&{0.68}$\pm$0.190&\bf{0.90}$\pm$0.135&0.75$\pm$0.123&\bf{0.94}$\pm$0.006&0.92$\pm$0.038&0.88$\pm$0.084\\
$T_{adv}^{2}$&0.01$\pm$0.051&{0.71}$\pm$0.251&{0.21}$\pm$0.357&\bf{0.87}$\pm$0.116&0.27$\pm$0.119&\bf{0.79}$\pm$0.069&0.68$\pm$0.127&0.56$\pm$0.249\\\midrule

\multicolumn{2}{l}{\emph{LavaCrossing}}\\
$\emptyset$&\bf{0.95}$\pm$0.023&{0.93}$\pm$0.050&{0.93}$\pm$0.018&0.88$\pm$0.091&{0.91}$\pm$0.024&{0.91}$\pm$0.063&{0.90}$\pm$0.017&\bf{0.92}$\pm$0.034\\
$T_1$&0.50$\pm$0.110&\bf{0.92}$\pm$0.053&{0.89}$\pm$0.029&0.64$\pm$0.109&0.66$\pm$0.071&\bf{0.78}$\pm$0.111&0.72$\pm$0.073&0.76$\pm$0.098\\
$T_2$&0.84$\pm$0.061&\bf{0.92}$\pm$0.050&\bf{0.92}$\pm$0.021&0.85$\pm$0.094&0.78$\pm$0.054&{0.83}$\pm$0.105&{0.86}$\pm$0.029&\bf{0.87}$\pm$0.063\\
$T_{adv}^{0.5}$&0.29$\pm$0.098&\bf{0.91}$\pm$0.081&\bf{0.91}$\pm$0.054&0.87$\pm$0.045&0.56$\pm$0.039&{0.51}$\pm$0.089&0.43$\pm$0.041&\bf{0.68}$\pm$0.126\\
$T_{adv}^{1}$&0.03$\pm$0.022&{0.83}$\pm$0.122&{0.86}$\pm$0.132&\bf{0.87}$\pm$0.059&0.27$\pm$0.158&{0.25}$\pm$0.118&0.17$\pm$0.067&\bf{0.43}$\pm$0.060\\
$T_{adv}^{2}$&0.0$\pm$0.004&{0.50}$\pm$0.171&{0.38}$\pm$0.020&\bf{0.82}$\pm$0.072&0.06$\pm$0.056&{0.04}$\pm$0.030&0.01$\pm$0.008&\bf{0.09}$\pm$0.060\\\midrule

\multicolumn{2}{l}{\emph{DynamicObstacles}}\\
$\emptyset$&\bf{0.91}$\pm$0.002&\bf{0.91}$\pm$0.008&\bf{0.91}$\pm$0.007&\bf{0.91}$\pm$0.131&\bf{0.91}$\pm$0.011&0.88$\pm$0.020&0.89$\pm$0.009&\bf{0.91}$\pm$0.013\\
$T_1$&0.23$\pm$0.201&\bf{0.77}$\pm$0.102&0.61$\pm$0.119&0.45$\pm$0.188&0.27$\pm$0.104&0.43$\pm$0.108&0.45$\pm$0.162&\bf{0.56}$\pm$0.270\\
$T_2$&0.50$\pm$0.117&\bf{0.75}$\pm$0.075&0.70$\pm$0.072&0.68$\pm$0.490&0.45$\pm$0.086&0.53$\pm$0.109&{0.52}$\pm$0.161&\bf{0.67}$\pm$0.203\\
$T_{adv}^{0.5}$&0.74$\pm$0.230&{0.89}$\pm$0.118&0.85$\pm$0.061&\bf{0.90}$\pm$0.142&0.46$\pm$0.214&{0.55}$\pm$0.197&0.51$\pm$0.371&\bf{0.62}$\pm$0.249\\
$T_{adv}^{1}$&0.26$\pm$0.269&{0.79}$\pm$0.157&{0.68}$\pm$0.144&\bf{0.84}$\pm$0.150&0.19$\pm$0.284&\bf{0.35}$\pm$0.197&0.23$\pm$0.370&0.10$\pm$0.379\\
$T_{adv}^{2}$&-0.49$\pm$0.312&{0.51}$\pm$0.234&{0.33}$\pm$0.202&\bf{0.55}$\pm$0.170&-0.54$\pm$0.209&-{0.21}$\pm$0.192&-0.53$\pm$0.261&\bf{-0.51}$\pm$0.260\\
\bottomrule
\end{tabular}
\caption{Extended Reward Results.}
\label{tab:app}
\end{table*}

This type of adversarial attack with SGLD optimiser was proposed in \citet{zhang2020robust}. As one can see, the adversarial disturbance is quite successful at severely lowering the obtained rewards in all scenarios. Additionally, as expected SA-PPO was the most effective at minimizing the disturbance effect (as it is trained with adversarial disturbances), although LRPG produces reasonably robust policies against this type of disturbances as well. At last, A2C appears to be much more sensitive to adversarial disturbances than PPO, indicating that the policies produced by PPO are by default more robust than A2C.

 \begin{figure}[htbp]
    \centering
    \includegraphics[width=0.49\linewidth]{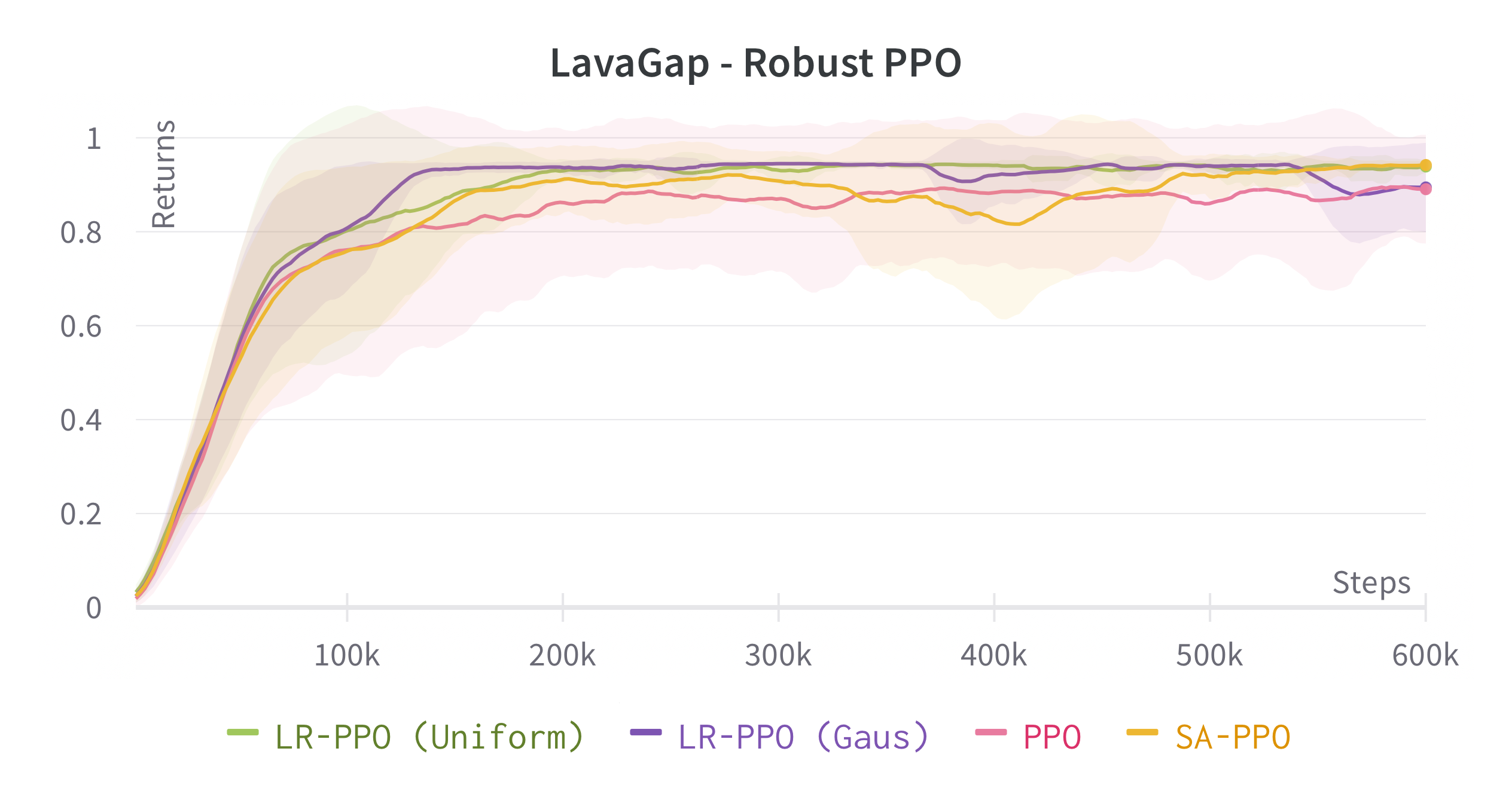}
    \includegraphics[width=0.49\linewidth]{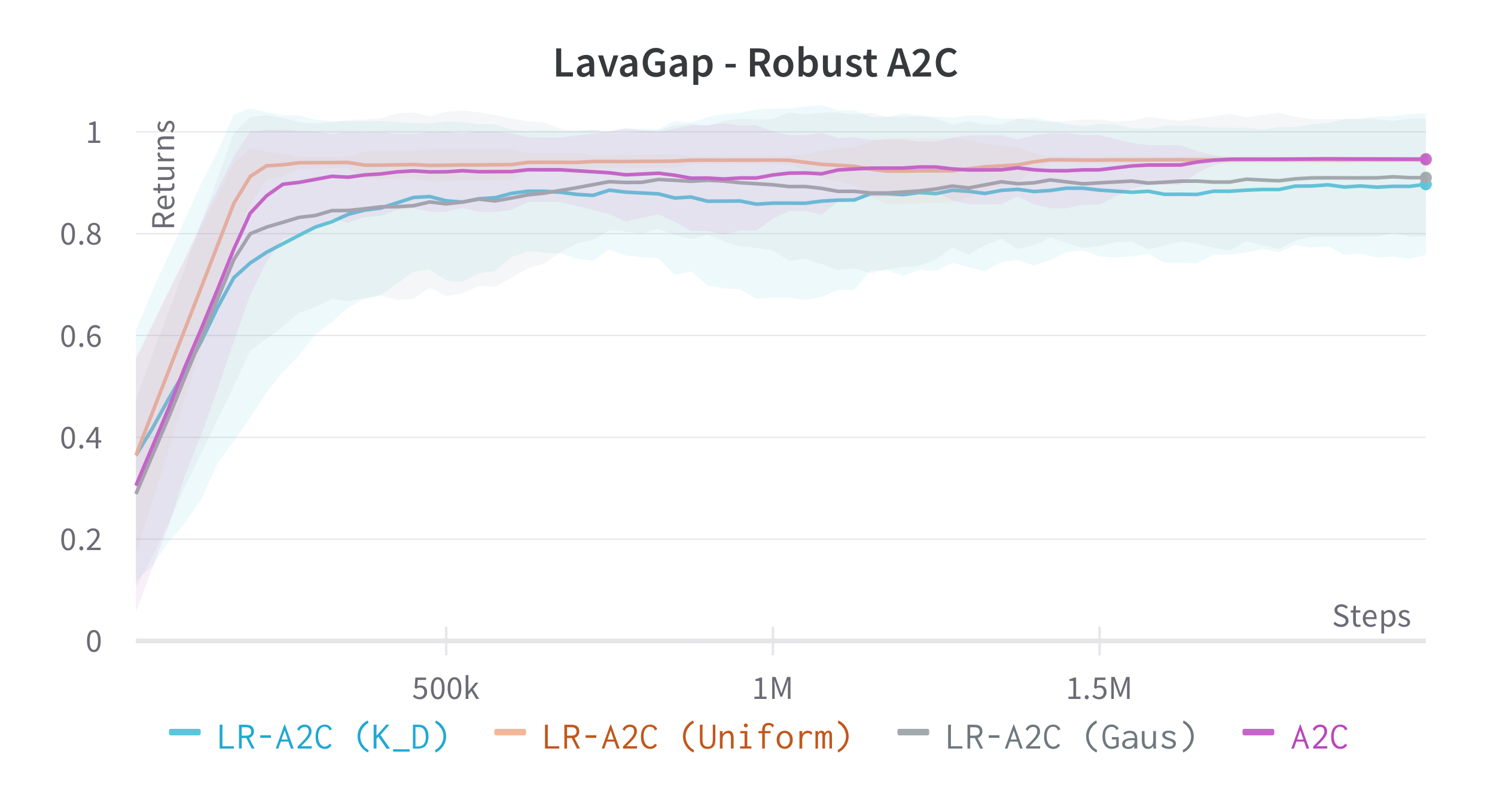}
    \includegraphics[width=0.49\linewidth]{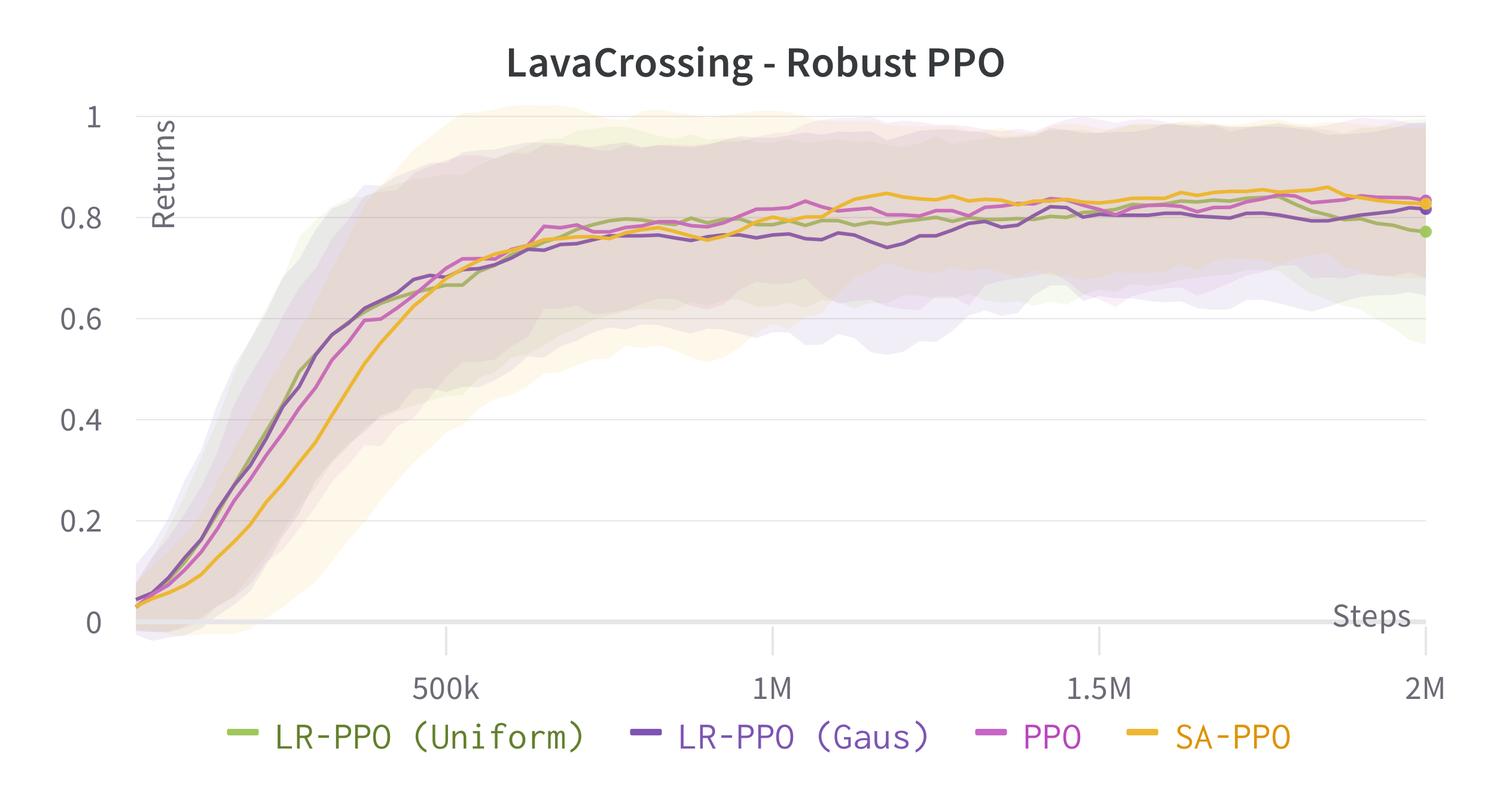}
    \includegraphics[width=0.49\linewidth]{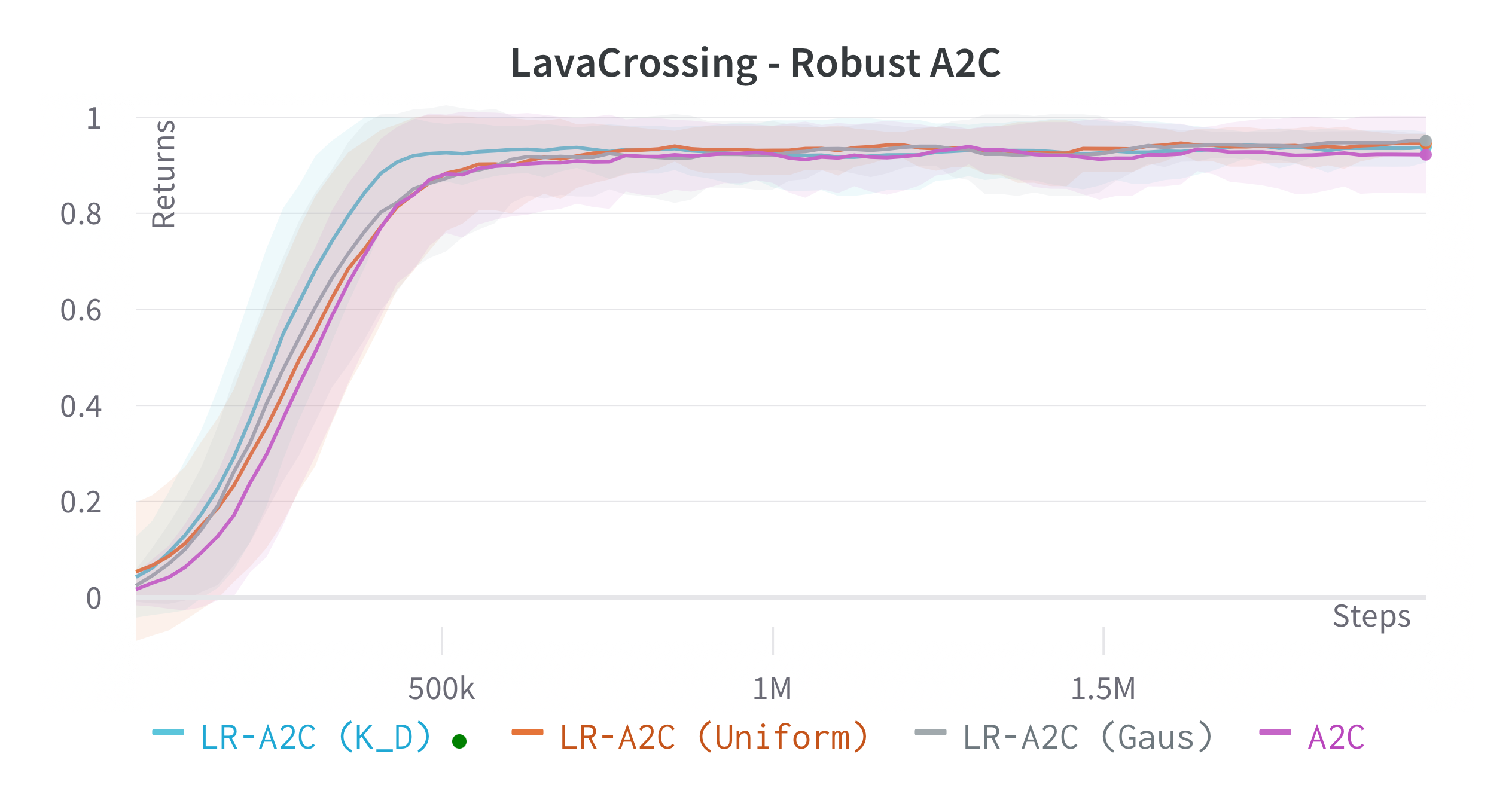}
    \includegraphics[width=0.49\linewidth]{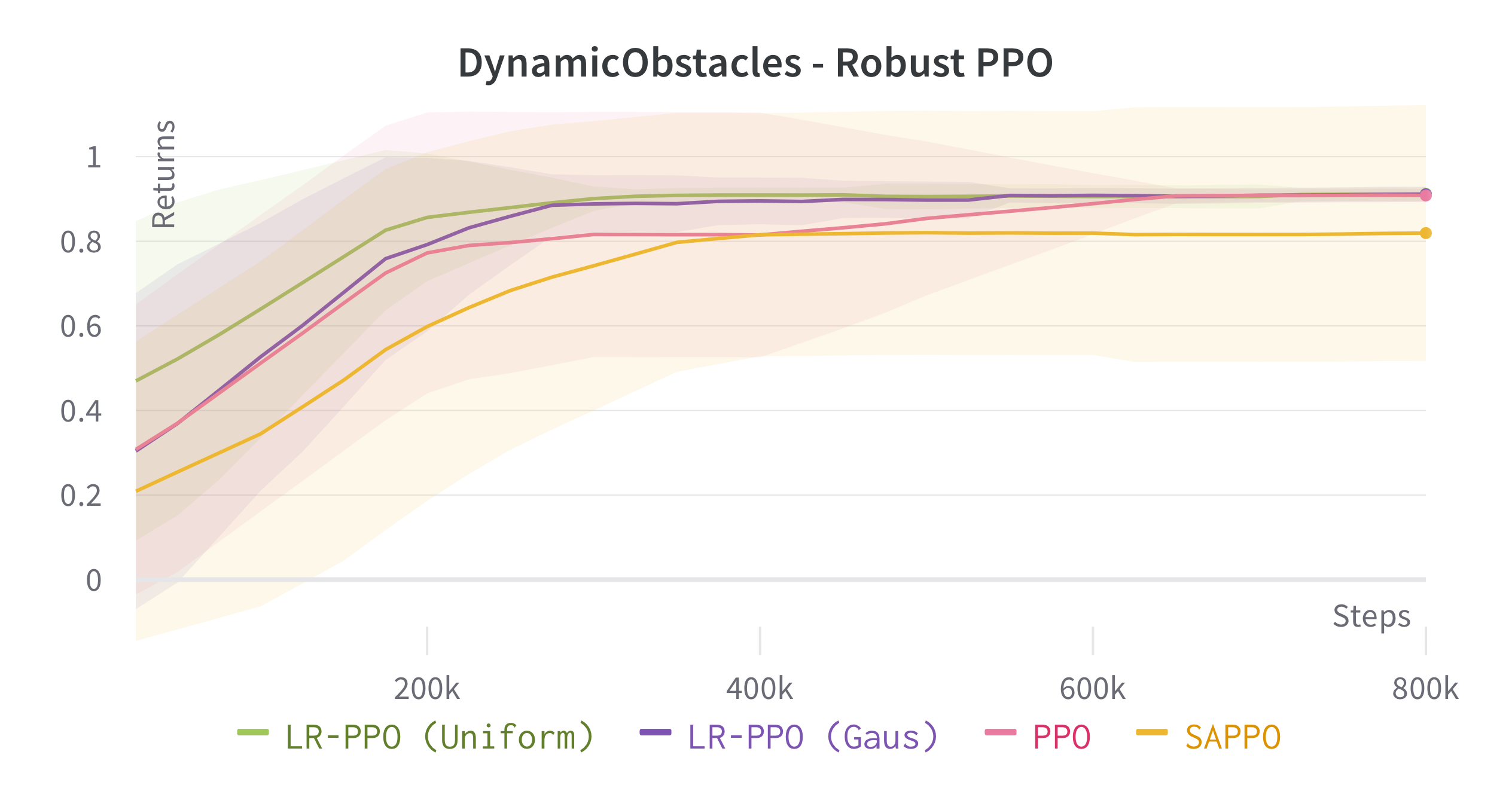}
    \includegraphics[width=0.49\linewidth]{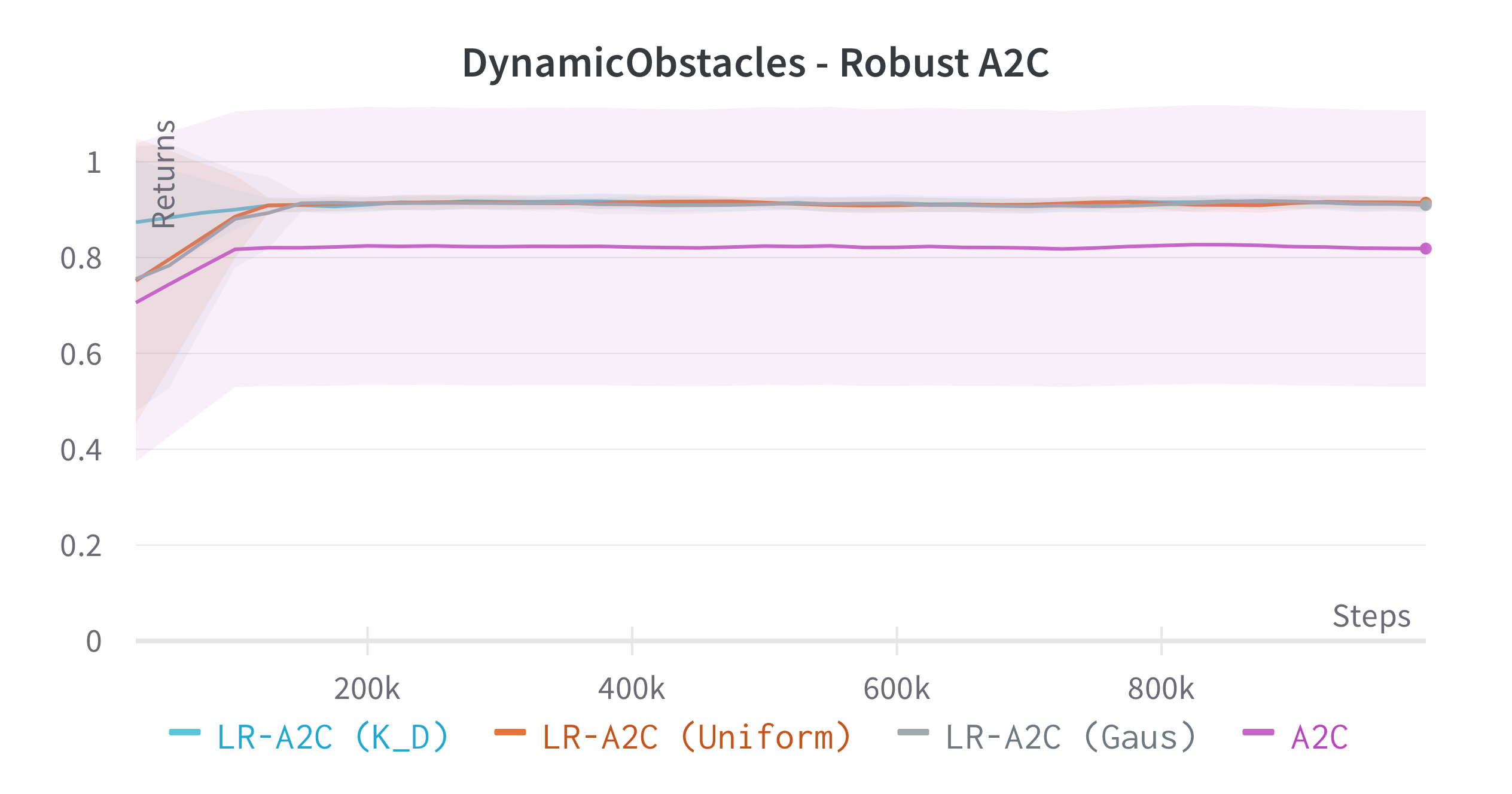}
    \caption{Learning Plots for Discrete Control Environments.}
\end{figure}
\subsection{Continuous Control}
The continuous control environments simulated are MountainCar, LunarLander and BipedalWalker. The policies used are in all cases MLP policies with ReLU gates and a $(64,64)$ feature extractor plus a fully connected layer to output the values and actions unless stated otherwise. The hyperparameters can be found in tables \ref{tab:ppocontinuous} and \ref{tab:saccontinuous}. The implementation is based on Stable Baselines 3 \citep{stable-baselines3} tuned algorithms. 
\begin{table*}[t] \centering
\tiny
\tabcolsep=0.11cm
\begin{tabular}{@{}lrrrr||rrr@{}}\toprule
& \multicolumn{4}{c}{\emph{PPO on Continuous Environments}} & \multicolumn{3}{c}{\emph{SAC on Continuous Environments}} 
\\ \cmidrule{1-8}
Noise & Vanilla & LR$_{\text{PPO}}$ $(K^u_{T})$ & LR$_{\text{PPO}}$ $(K^g_{T})$ & SA-PPO & Vanilla & LR$_{\text{SAC}}$ $(K^u_{T})$ & LR$_{\text{SAC}}$ $(K^g_{T})$ \\ \midrule
\multicolumn{2}{l}{\emph{MountainCar}}\\
$\emptyset$& \bf{94.77}$\pm$0.26 &{93.17}$\pm$0.89 & {94.66}$\pm$1.61 & 88.69$\pm$3.93 & 93.52$\pm$0.05 & \bf{94.43}$\pm$0.19 & {93.84}$\pm$0.05 \\
$T_1$&88.67$\pm$1.41&{91.46}$\pm$1.22 & \bf{94.91}$\pm$1.35 & 88.41$\pm$3.99& 1.89$\pm$65.31&{71.81}$\pm$13.04&\bf{76.90}$\pm$7.11 \\
$T_2$&92.22$\pm$1.11&{92.40}$\pm$1.28&\bf{94.76}$\pm$1.42 & 89.32$\pm$3.79 & -27.82$\pm$73.10&\bf{72.93}$\pm$8.57&69.41$\pm$13.03 \\\midrule
\multicolumn{2}{l}{\emph{LunarLander}}\\
$\emptyset$&{267.99}$\pm$38.04 &\bf{269.76}$\pm$22.93 &{243.08}$\pm$37.03 & 220.18$\pm$98.78 &{268.96}$\pm$51.52&{275.17}$\pm$14.04&\bf{282.24}$\pm$15.95 \\
$T_1$&156.09$\pm$22.87&\bf{280.91}$\pm$20.34&{182.80}$\pm$49.26 & 164.53 $\pm$45.48&128.18$\pm$17.73&\bf{187.64}$\pm$76.30&{153.81}$\pm$33.16 \\
$T_2$&158.02$\pm$46.57&\bf{276.76}$\pm$16.20&{212.62}$\pm$37.56 & 221.84$\pm$73.61 &{140.92}$\pm$20.61&\bf{187.82}$\pm$25.27&{158.18}$\pm$28.60 \\\midrule
\multicolumn{2}{l}{\emph{BipedalWalker}}\\
$\emptyset$&{265.39}$\pm$82.36&{261.39}$\pm$83.19&\bf{276.66}$\pm$44.85 & 251.60$\pm$103.08 &236.39 $\pm$157.03&{302.56}$\pm$70.79&\bf{313.56}$\pm$52.17 \\
$T_1$&174.15$\pm$170.30&{253.56}$\pm$72.66& 220.28$\pm$118.61&\bf{264.69}$\pm$61.63 &203.93 $\pm$167.83&241.45$\pm$124.54&\bf{241.60}$\pm$139.93 \\
$T_2$&135.16$\pm$182.30&{243.27}$\pm$89.86& \bf{265.37}$\pm$80.60&255.21$\pm$90.61 &84.10 $\pm$198.12&198.20$\pm$151.64&\bf{229.75}$\pm$166.87 \\
\bottomrule
\end{tabular}
\caption{Reward values gained by LRPG and baselines on continuous control tasks.}
\label{tab:4}

\end{table*}
\textbf{Noise Kernels.} We consider again two types of noise; a normal distributed noise $\tilde{T}^g$ and a uniform distributed noise $\tilde{T}^u$. In all cases, algorithms are implemented with a state observation normalizer. That is, assimptotically all states will be observed to be in the set $(-1,1)$. For this reason, the uniform noise is bounded at lower values than for the discrete control environments. For BipedalWalker $\|\xi\|_{\infty}\leq 0.05$ and for Lunarlander and MountainCar $\|\xi\|_{\infty}\leq 0.1$. Larger values were shown to destabilize learning. 
\begin{table}
\centering
\begin{tabular}{||l c c c||}
\hline
& MountainCarContinuous & LunarLanderContinuous & BipedalWalker-v3\\
\hline\hline
Parallel Envs & 1 & 16 & 32\\
\hline
Steps & $2\times 10^4$ & $1\times 10^6$ & $5\times 10^6$ \\
\hline
$\gamma$ & 0.9999 & 0.999 & 0.999 \\
\hline
$\alpha$ & $3\times 10^{-4}$ & $3\times 10^{-4}$ & $3\times 10^{-4}$ \\
\hline
Grad. Clip & 5 & 0.5 & 0.5 \\
\hline
Ratio Clip & 0.2 & 0.2 & 0.18\\
\hline
Gae & 0.9 & 0.98 & 0.95 \\
\hline
Epochs & 10 & 4 & 10 \\
\hline
E. Coeff & 0.00429 & 0.01 & 0 \\
\hline
\end{tabular}
\label{tab:ppocontinuous}
\caption{PPO Parameters for Continuous Control}
\end{table}

\begin{table}
\centering
\begin{tabular}{||l c c c||}
\hline
& MountainCarContinuous & LunarLanderContinuous & BipedalWalker-v3\\
\hline\hline
Steps & $5\times 10^4$ & $5\times 10^5$ & $5\times 10^5$ \\
\hline
$\gamma$ & 0.9999 & 0.99 & 0.98 \\
\hline
$\alpha$ & $3\times 10^{-4}$ & $7.3\times 10^{-4}$ & $7.3\times 10^{-4}$ \\
\hline
$\tau$ & 0.01 & 0.01 & 0.01 \\
\hline
Train Freq. & 32 & 1 & 64 \\
\hline
Grad. Steps & 32 & 1 & 64 \\
\hline
MLP Arch & (64,64) & (400,300) & (400,300) \\
\hline
\end{tabular}
\caption{SAC Parameters for Continuous Control}
\label{tab:saccontinuous}
\end{table}

\begin{figure}[htbp]
    \centering
        \includegraphics[width=0.49\linewidth]{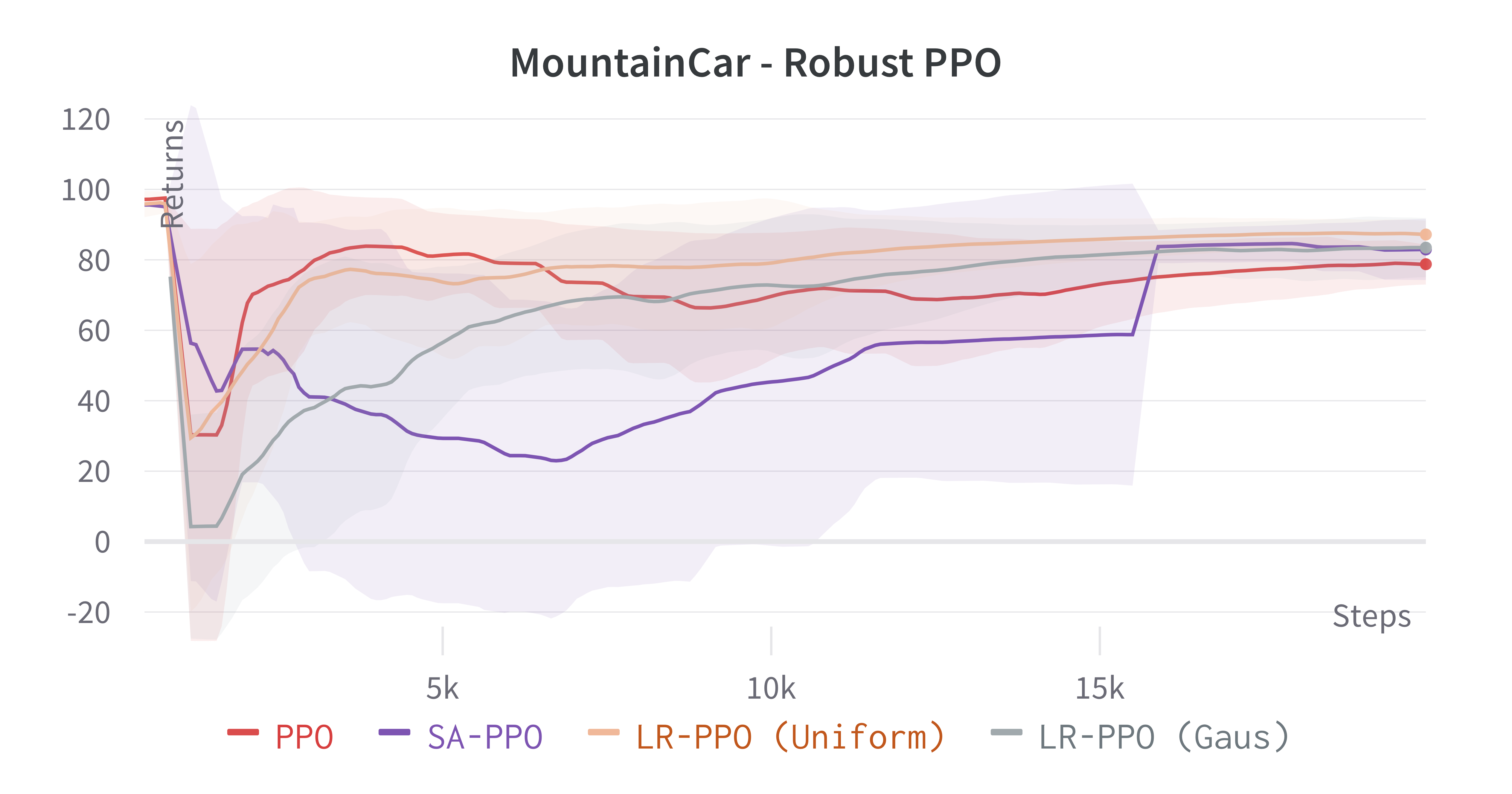}
        \includegraphics[width=0.49\linewidth]{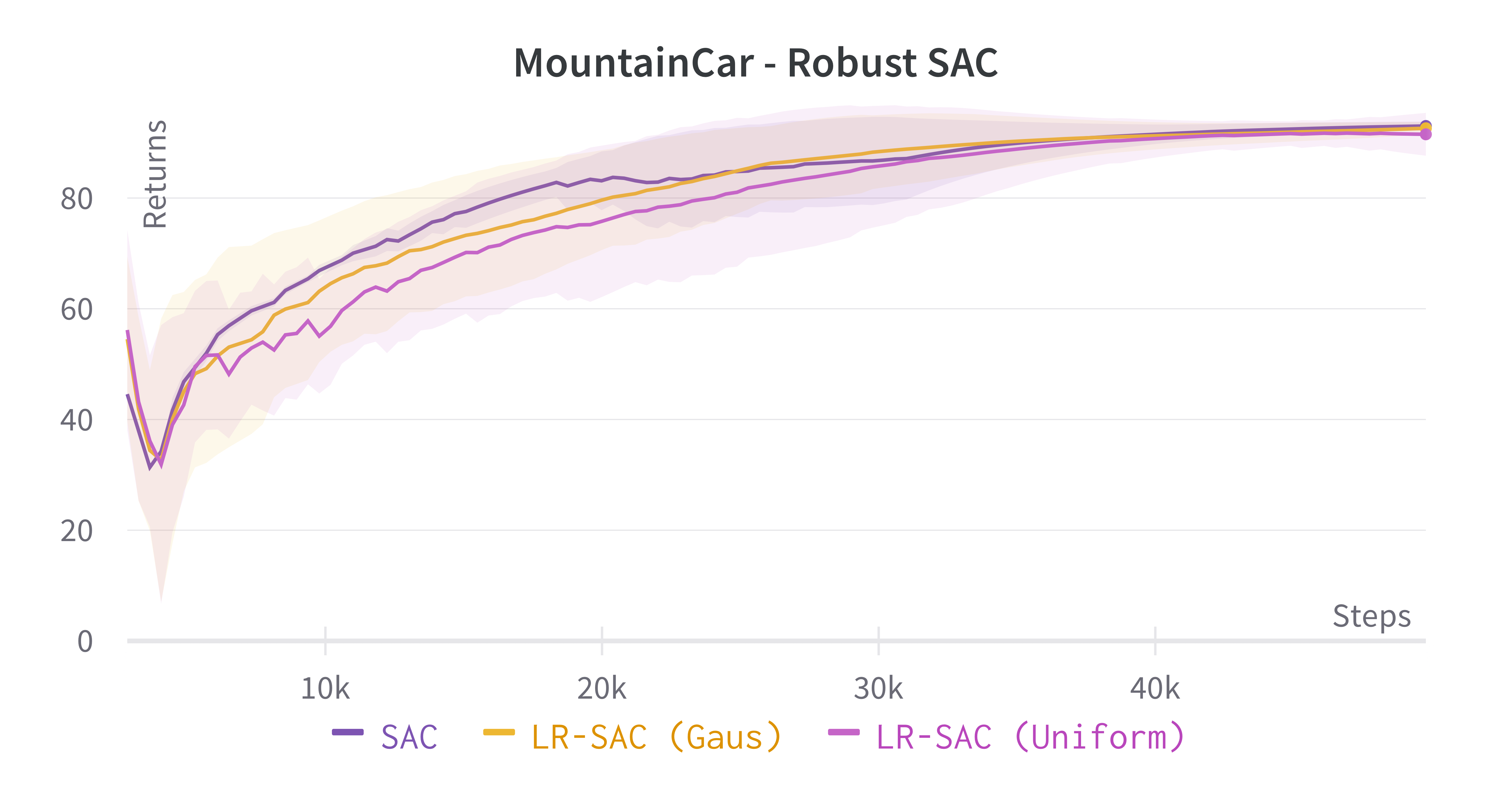}

        \includegraphics[width=0.49\linewidth]{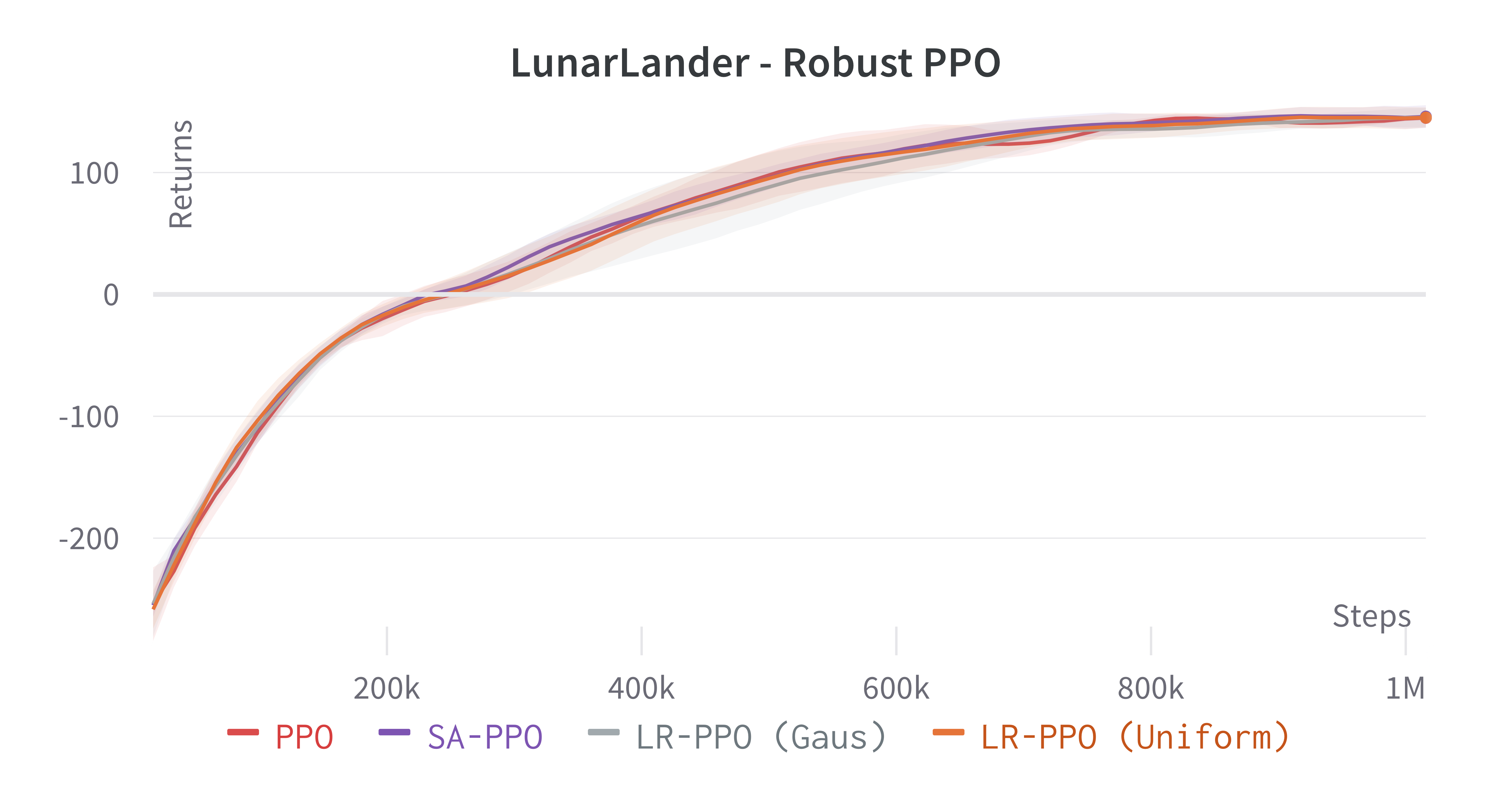}

        \includegraphics[width=0.49\linewidth]{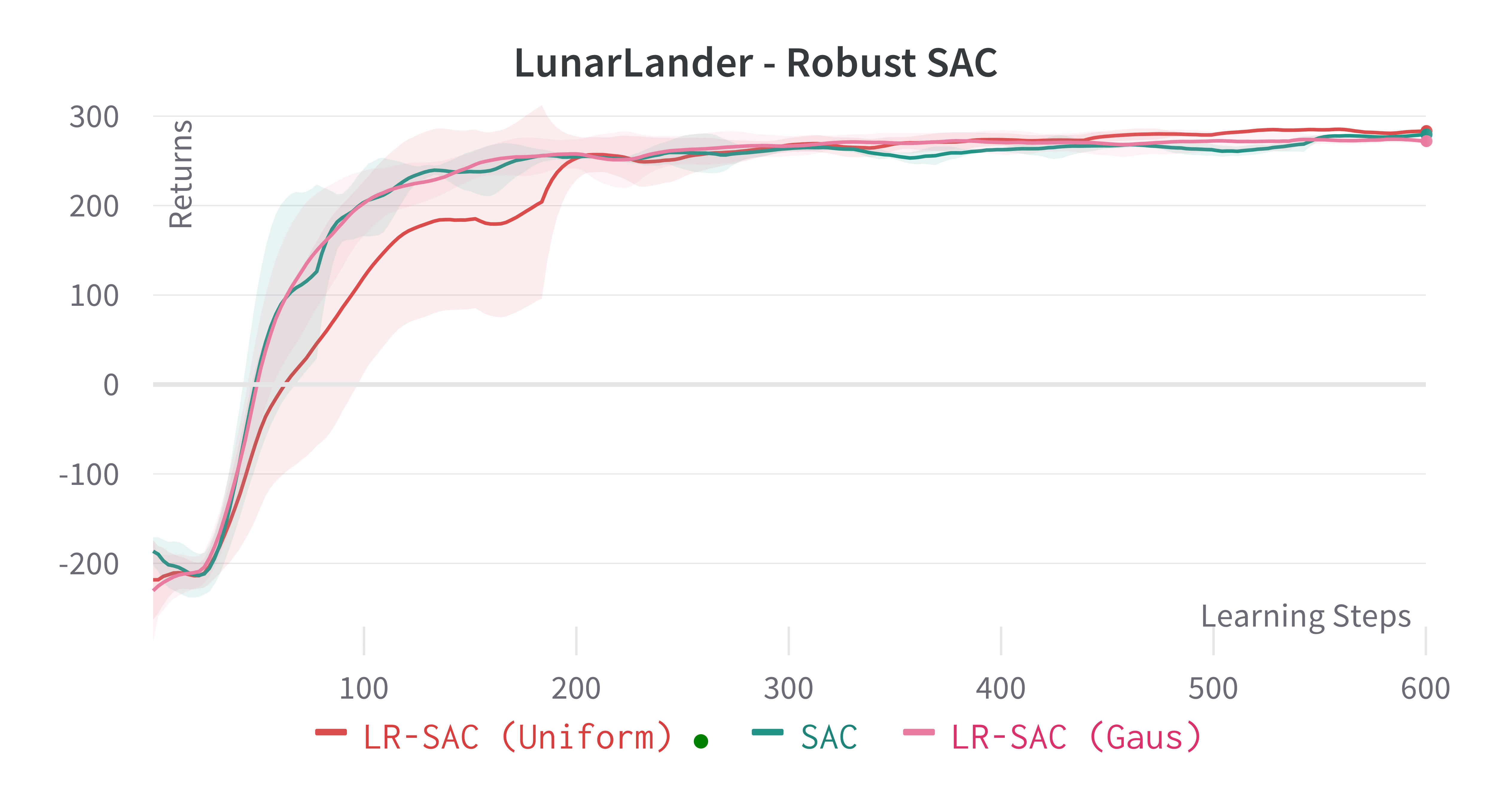}
        \includegraphics[width=0.49\linewidth]{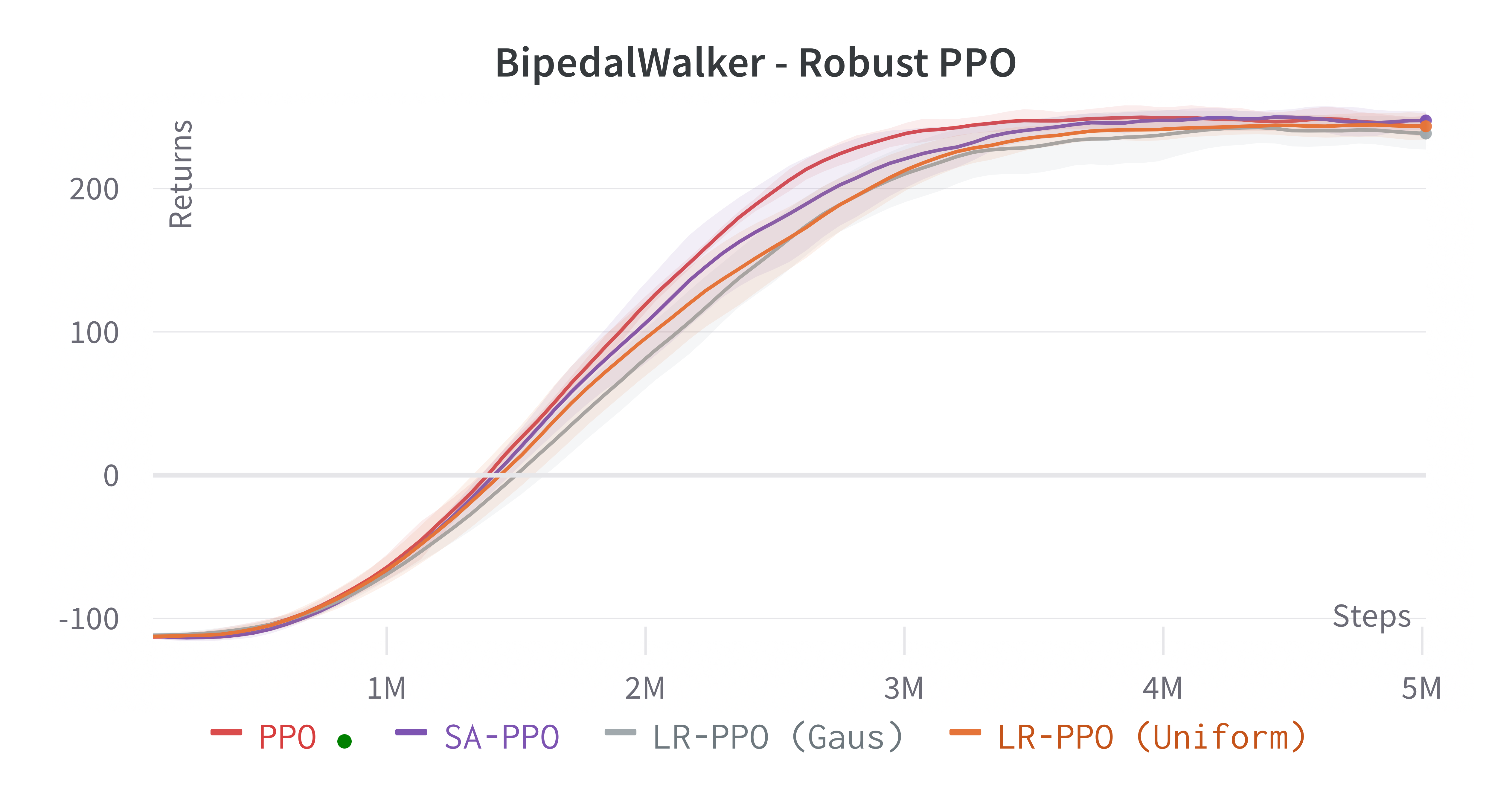}
        \includegraphics[width=0.49\linewidth]{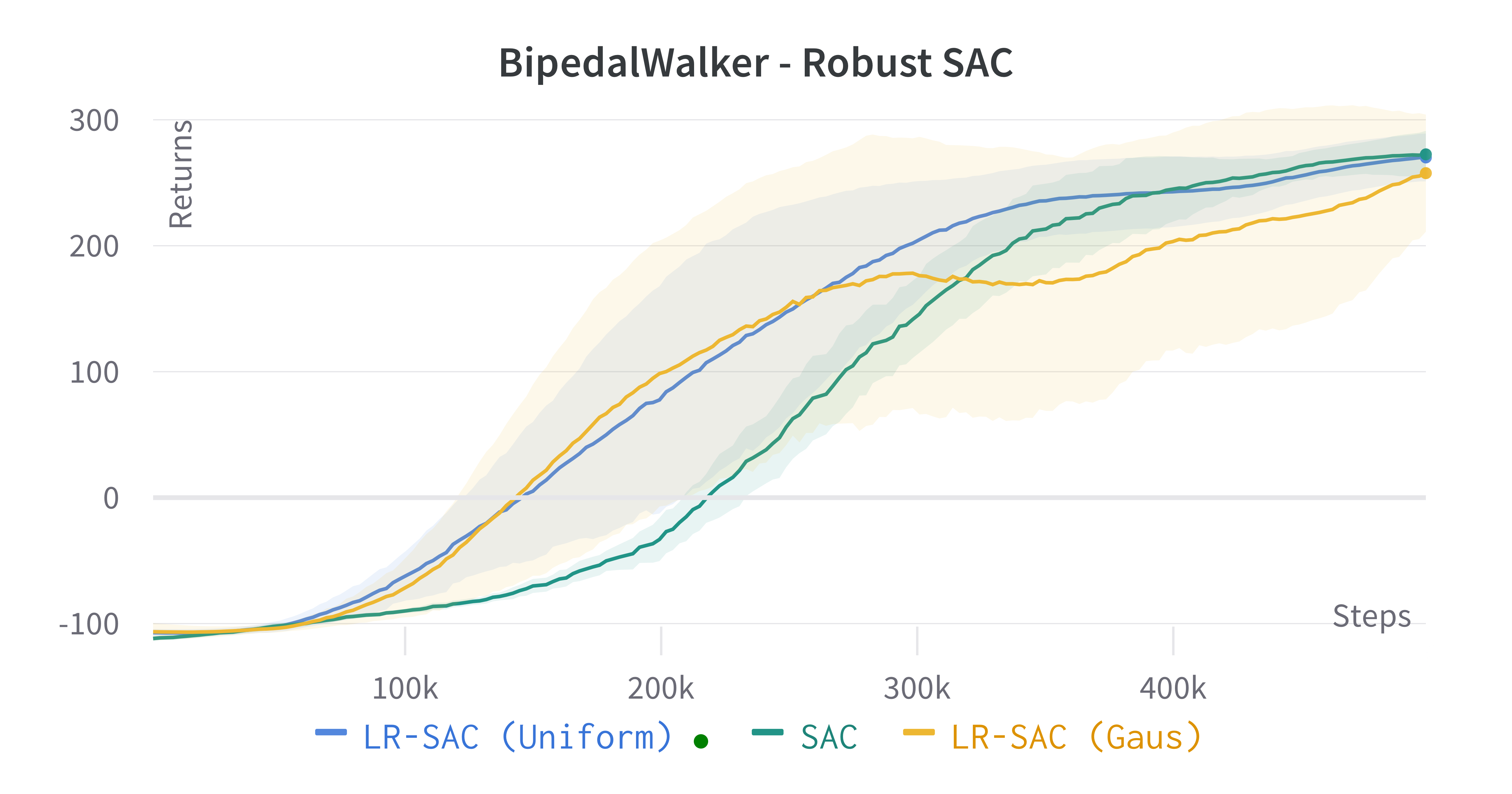}
    \caption{Learning Plots for Continuous Control Environments.}
\end{figure}
\paragraph{Learning processes} In general, learning was not severlely affected by the LRPG scheme. However, it was shown to induce a larger variance in the trajectories observed, as seen in LunarLander with LR-SAC and BipedalWalker with LR-SAC.
\end{document}